\documentclass[letterpaper,10pt]{article}
\usepackage{aaai}
\usepackage{times}
\usepackage{helvet}
\usepackage{courier}
\usepackage{amsthm}
\usepackage{amssymb}
\usepackage{amsmath}
\usepackage{graphicx} 
\usepackage{graphics, color}
\usepackage{algorithm}
\usepackage{algpseudocode}
\usepackage{url}
\newfloat{algorithm}{t}{lop}
\usepackage{enumitem}
\newcommand{\ceil}[1]{\left\lceil #1 \right\rceil}
\newcommand{\prob}{\ensuremath{\mathsf{Pr}}}
\newcommand{\expect}{\ensuremath{\mathsf{E}}}

\newcommand{\satisfying}[1]{\ensuremath{R_{#1}}} %
\newcommand{\weight}[1]{\ensuremath{w \! \left( #1 \right)}} %
\newcommand{\scaledweight}[1]{\ensuremath{\mathcal{W} \left( #1 \right)}} %
\newcommand{\wmax}{\ensuremath{w_{max}}} %
\newcommand{\wmin}{\ensuremath{w_{min}}} %
\newcommand{\ratio}{\ensuremath{\rho}} %

\newcommand{\UniformWitness}{\ensuremath{\mathsf{UniWit}}}
\newcommand{\UniGen}{\ensuremath{\mathsf{UniGen}}}
\newcommand{\PAWS}{\ensuremath{\mathsf{PAWS}}}

\newcommand{\var}{\ensuremath{\mathsf{V}}}

\newcommand{\NP}{\ensuremath{\mathsf{NP}}}

\newcommand{\SAT}{\ensuremath{\mathsf{SAT}}}
\newcommand{\UNSAT}{\ensuremath{\mathsf{UNSAT}}}

\newcommand{\Cachet}{\ensuremath{\mathsf{Cachet}}}

\newcommand{\WeightMC}{\ensuremath{\mathsf{WeightMC}}}
\newcommand{\PartitionedWeightMC}{\ensuremath{\mathsf{PartitionedWeightMC}}}
\newcommand{\ApproxMC}{\ensuremath{\mathsf{ApproxMC}}}
\newcommand{\WeightMCCore}{\ensuremath{\mathsf{WeightMCCore}}}

\newcommand{\ComputeKappaAndPivot}{\ensuremath{\mathsf{ComputeKappaPivot}}}

\newcommand{\BoundedWeightSAT}{\ensuremath{\mathsf{BoundedWeightSAT}}}

\newcommand{\ComputeIterCount}{\ensuremath{\mathsf{ComputeIterCount}}}

\newcommand{\WeightGen}{\ensuremath{\mathsf{WeightGen}}}
\newcommand{\AddBlockClause}{\ensuremath{\mathsf{AddBlockClause}}}
\newcommand{\SolveSAT}{\ensuremath{\mathsf{SolveSAT}}}
\newcommand{\SDD}{\ensuremath{\mathsf{SDD}}}
\newcommand{\IS}{\ensuremath{\mathsf{IS}}}
\newcommand{\MAP}{\ensuremath{\mathsf{MAP}}}
\newcommand{\WISH}{\ensuremath{\mathsf{WISH}}}

\newcommand{\killthis}[1]{}

\newtheoremstyle{newstyle}      
  {.5\baselineskip\@plus.2\baselineskip\@minus.2\baselineskip}%
  {.5\baselineskip\@plus.2\baselineskip\@minus.2\baselineskip}%
{\mdseries} %
{} %
{\bfseries} %
{.} %
{ } %
{} %

\newtheorem{lemma}{Lemma}
\newtheorem{theorem}{Theorem}

\algrenewcomment[1]{\(\triangleright\) #1}

\algnotext{EndFor}
\algnotext{EndIf}

\everymath{\displaystyle}

\title{ Distribution-Aware Sampling and Weighted Model Counting for SAT}
\author{Supratik Chakraborty\\Indian Institute of Technology, Bombay
\And Daniel J. Fremont\\University of California, Berkeley
\AND Kuldeep S. Meel\\Department of Computer Science, Rice University
\And Sanjit A. Seshia\\University of California, Berkeley
\AND Moshe Y. Vardi\\Department of Computer Science, Rice University
}
\begin{document}
\maketitle

\begin{abstract}
Given a CNF formula and a weight for each assignment of values to
variables, two natural problems are weighted model counting and
distribution-aware sampling of satisfying assignments.  Both problems 
have a wide variety of important applications.  Due to the inherent
complexity of the exact versions of the problems, interest has focused
on solving them approximately.  Prior work in this area scaled only to
small problems in practice, or failed to provide strong theoretical
guarantees, or employed a computationally-expensive maximum a 
posteriori probability (MAP) oracle that assumes prior knowledge of a
factored representation of the weight distribution.  We present a
novel approach that works with a black-box oracle for weights of
assignments and requires only an {\NP}-oracle (in practice, a SAT-solver) to solve both the
counting and sampling problems.  Our approach works 
under mild assumptions on the distribution of weights of satisfying
assignments, provides strong theoretical guarantees, and scales to
problems involving several thousand variables. We also show that the
assumptions can be significantly relaxed while improving computational efficiency 
if a factored representation of the weights is known. 
\end{abstract}

\section{Introduction}
Given a set of weighted elements, computing the cumulative weight of
all elements that satisfy a set of constraints is a fundamental
problem that arises in many contexts.  Known variously as weighted
model counting, discrete integration and partition function
computation, this problem has applications in machine learning,
probabilistic reasoning, statistics, planning and combinatorics, among
other areas ~\cite{Roth1996,Sang04combiningcomponent,domshlak2007,xue2012basing}.  
A closely related problem is that of sampling
elements satisfying a set of constraints, where the probability of
choosing an element is proportional to its weight.  The latter
problem, known as weighted sampling, also has important applications
in probabilistic reasoning, machine learning, statistical physics,
constrained random verification and other
domains~\cite{jerrum1996markov,Bacchus2003,NavRim07,madras1999importance}.  Unfortunately, the exact
versions of both problems are computationally hard.  Weighted model
counting can be used to count the number of satisfying assignments of
a CNF formula; hence it is $\#P$-hard~\cite{valiant1979complexity}.  It is also
known that an efficient algorithm for weighted sampling would yield a
fully polynomial randomized approximation scheme (FPRAS) for
{\#P}-complete inference problems~\cite{jerrum1996markov,madras1999importance} -- a
possibility that lacks any evidence so far.  Fortunately, approximate
solutions to weighted model counting and weighted sampling are good
enough for most applications.  Consequently, there has been
significant interest in designing practical approximate algorithms for
these problems.

Since constraints arising from a large class of real-world problems
can be modeled as propositional CNF (henceforth CNF) formulas, we
focus on CNF and assume that the weights of truth assignments are
given by a weight function $w(\cdot)$ defined on the set of truth
assignments.  Roth showed that approximately counting the models of a
CNF formula is {\NP}-hard even when the structure of the formula is
severely restricted~\cite{Roth1996}.  By a result of Jerrum, Valiant
and Vazirani ~\cite{Jerr}, we also know that approximate model
counting and almost uniform sampling (a special case of approximate
weighted sampling) are polynomially inter-reducible.  Therefore, it is
unlikely that there exist polynomial-time algorithms for either
approximate weighted model counting or approximate weighted
sampling~\cite{KarpLuby1989}.  Recently, a new class of algorithms
that use pairwise independent random parity constraints and a {\MAP}
(\emph{maximum a posteriori probability})-oracle have been proposed
for solving both
problems~\cite{ermon2013icml,ermon2014icml,ermon2013nips}.  These
algorithms provide strong theoretical guarantees (FPRAS relative to
the {\MAP} oracle), and have been shown to scale to medium-sized
problems in practice.  While this represents a significant step in our
quest for practically efficient algorithms with strong guarantees for
approximate weighted model counting and approximate weighted sampling,
the use of {\MAP}-queries presents issues that need to be addressed in
practice.  First, the use of {\MAP}-queries along with parity
constraints poses scalability
hurdles~\cite{ermon2014icml,ermon2013uai}.  Second, existing
{\MAP}-query solvers work best when the distribution of weights is
represented by a graphical model with small tree-width -- a
restriction that is violated in several real-life problems.  While
this does not pose problems in practical applications where an
approximation of the optimal {\MAP} solution without guarantees of the
approximation factor suffices, it presents significant challenges when
we demand the optimal MAP solution.  This motivates us to ask if we
can design approximate algorithms for weighted model counting and
weighted sampling that do not invoke {\MAP}-oracles at all, and do not
assume any specific representation of the weight distribution.

Our primary contribution is an affirmative answer to the above
question under mild assumptions on the distribution of weights.
Specifically, we show that two recently-proposed algorithms for
approximate (unweighted) model counting~\cite{CMV13MC} and 
near-uniform (unweighted) sampling~\cite{CMV-CAV13} can be adapted to work in the
setting of weighted assignments, using only a SAT solver
({\NP}-oracle) and a black-box weight function $w(\cdot)$.  For the
algorithm to work well in practice, we require that \emph{tilt} of
the weight function, which is the  ratio of the maximum weight 
of a satisfying assignment to the minimum weight of a satisfying assignment, 
is small.  We present arguments why this is a reasonable assumption 
in some important classes of problems.  We also present an adaptation of our 
algorithm for problem instances where the tilt is large.
The adapted algorithm requires a pseudo-Boolean SAT solver instead of
a (regular) SAT solver as an oracle.

\section{Notation and Preliminaries}\label{sec:prelims}

Let $F$ be a Boolean formula in conjunctive normal form (CNF), and let
$X$ be the set of variables appearing in $F$.  The set $X$ is called
the \emph{support} of $F$.  Given a set of variables $S \subseteq X$ 
and an assignment $\sigma$ of truth values to the variables in $X$, 
we write $\sigma|_{S}$ for the projection of $\sigma$ onto $S$. 
A \emph{satisfying assignment} or \emph{witness} of $F$ is an assignment that 
makes $F$ evaluate to true.  We denote the set of all witnesses of $F$ by $\satisfying{F}$. 
For notational convenience, whenever the formula $F$ is clear from the context, 
we omit mentioning it. Let ${\mathcal D} \subseteq X$ be a subset of the support such that there 
are no two satisfying assignments that differ only in the truth values of 
variables in ${\mathcal D}$.  In other words, in every satisfying assignment, 
the truth values of variables in $X\setminus {\mathcal D}$ uniquely determine 
the truth value of every variable in ${\mathcal D}$.  The
set ${\mathcal D}$ is called a \emph{dependent} support of $F$, and 
$X\setminus \mathcal{D}$ is called an \emph{independent} support.  
Note that there may be more than one independent support: 
$(a \vee \neg b) \wedge (\neg a \vee b)$ has three, namely $\{a\}$, $\{b\}$ and $\{a, b\}$. 
Clearly, if ${\mathcal{I}}$ is an independent support of $F$, so is every superset 
of ${\mathcal{I}}$.

Let $\weight{\cdot}$ be a function that takes as input an assignment
$\sigma$ and yields a real number $\weight{\sigma} \in (0, 1]$ called
  the \emph{weight} of $\sigma$.  Given a set $Y$ of assignments, we
  use $\weight{Y}$ to denote $\Sigma_{\sigma \in Y} \weight{\sigma}$.
  Our main algorithms (see Section~\ref{sec:algorithm}) make no
  assumptions about the nature of the weight function, treating it as
  a black-box function.  In particular, we do not assume that the
  weight of an assignment can be factored into the weights of
  projections of the assignment on specific subsets of variables.  The
  exception to this is Section \ref{sec:white-box}, where we consider
  possible improvements when the weights are given by a known
  function, or ``white-box''. Three important quantities derived from the weight function are 
$\wmax = \max_{\sigma \in \satisfying{F}} \weight{\sigma}$, 
$\wmin = \min_{\sigma \in \satisfying{F}} \weight{\sigma}$, 
and the \emph{tilt} $\ratio = \wmax / \wmin$.
Our algorithms require an 
upper bound on the tilt, denoted $r$, which is provided by the user.
As tight a bound as possible is desirable to maximize the efficiency of 
the algorithms.  We define {\MAP}
  (\emph{maximum a posteriori probability}) for our distribution of weights to be $\frac{w_{max}}{\weight{R_F}}$.  
  
  We write $\prob\left[X: {\cal P} \right]$ for the probability of
outcome $X$ when sampling from a probability space ${\cal P}$.  
For brevity, we omit ${\cal P}$ when it is clear from the
context.  The expected value of the outcome $X$ is denoted
$\expect\left[X\right]$.

A special class of hash functions, called
\emph{$k$-wise independent} hash functions, play a crucial role in
our work ~\cite{Bellare98uniformgeneration}.  Let $n, m$ and $k$ be positive integers, and let
$H(n,m,k)$ denote a family of $k$-wise independent hash functions
mapping $\{0, 1\}^n$ to $\{0, 1\}^m$.  We use $h \xleftarrow{R}
H(n,m,k)$ to denote the probability space obtained by choosing a
hash function $h$ uniformly at random from $H(n,m,k)$.  The property
of $k$-wise independence guarantees that for all $\alpha_1, \ldots
\alpha_k \in \{0,1\}^m $ and for all distinct $y_1, \ldots y_k \in
\{0,1\}^n$, $\prob\left[\bigwedge_{i=1}^k h(y_i) = \alpha_i\right.$
$\left.: h \xleftarrow{R} H(n, m, k)\right] = 2^{-mk}$.  For every
$\alpha \in \{0, 1\}^m$ and $h \in H(n, m, k)$, let $h^{-1}(\alpha)$
denote the set $\{y \in \{0, 1\}^n \mid h(y) = \alpha\}$.  Given
$\satisfying{F} \subseteq \{0, 1\}^n$ and $h \in H(n, m, k)$, we use 
$\satisfying{F, h, \alpha}$ to denote the set $\satisfying{F} \cap h^{-1}(\alpha)$.  

Our work uses an efficient family of hash functions, denoted as $H_{xor}(n, m, 3)$. 
Let $h: \{0, 1\}^n \rightarrow \{0, 1\}^m$ be a hash function in
the family, and let $y$ be a vector in $\{0, 1\}^n$.  Let $h(y)[i]$
denote the $i^{th}$ component of the vector obtained by applying $h$
to $y$.  The family of hash functions of interest is defined as
$\{h(y) \mid h(y)[i] = a_{i,0} \oplus (\bigoplus_{l=1}^n a_{i,l}\cdot
y[l]), a_{i,j} \in \{0, 1\}, 1 \leq i \le m, 0 \leq j \leq n\}$, where
$\oplus$ denotes the xor operation.  By choosing values of $a_{i,j}$
randomly and independently, we can effectively choose a random hash
function from the family.  It has been shown in~\cite{Gomes-Sampling}
that this family of hash functions is $3$-independent. 

Given a CNF formula $F$, an \emph{exact weighted model counter}
returns $\weight{\satisfying{F}}$.  An \emph{approximate weighted
  model counter} relaxes this requirement to some extent: given
\emph{tolerance} $\varepsilon > 0$ and \emph{confidence} $1-\delta \in
(0, 1]$, the value $v$ returned by the counter satisfies
  $\prob[\frac{\weight{\satisfying{F}}}{1+\varepsilon} \le v \le
    (1+\varepsilon)\weight{\satisfying{F}}] \ge 1-\delta$.  A related
  type of algorithm is a \emph{weighted-uniform probabilistic generator}, which
  outputs a witness $w \in \satisfying{F}$ such that $\prob\left[w =
    y\right] = \weight{y}/\weight{\satisfying{F}}$ for every $y \in
  \satisfying{F}$.  An \emph{almost weighted-uniform generator}
  relaxes this requirement, ensuring that for all $y \in
  \satisfying{F}$, we have $\frac{\weight{y}}{(1+
    \varepsilon)\weight{\satisfying{F}}}$ $\le$ $\prob\left[w =
    y\right]\le$ $\frac{(1 +
    \varepsilon)\weight{y}}{\weight{\satisfying{F}}}$.  Probabilistic
  generators are allowed to occasionally ``fail" by not returning a
  witness (when $\satisfying{F}$ is non-empty), with the failure
  probability upper bounded by $\delta$.

\section{Related Work}
\label{sec:relatedwork}

Marrying strong theoretical guarantees with scalable performance is
the holy grail of research in the closely related areas of weighted
model counting and weighted sampling.  The tension between the two
objectives is evident from a survey of the literature.  Earlier
algorithms for weighted model counting can be broadly divided into
three categories: those that give strong guarantees but scale poorly
in practice, those that give weak guarantees but scale well in
practice, and some recent algorithms that attempt to bridge this gap
by making use of a {\MAP}-oracle and random parity constraints.
Techniques in the first category attempt to compute the weighted model
count exactly by enumerating partial
solutions~\cite{SangBearKautz2005} or by converting the CNF formula to
alternative representations~\cite{darwiche2004new,ChoiDarwiche13}.
Unfortunately, none of these approaches scale to large problem
instances.  Techniques in the second category employ variational
methods, sampling-based methods or other heuristic methods.
Variational methods~\cite{wainwright2008graphical,GogDech2011} work
extremely well in practice, but do not provide guarantees except in
very special cases.  Sampling-based methods are usually based on
importance sampling (e.g. ~\cite{GogDech2011}), which provide weak
one-sided bounds, or on Markov Chain Monte Carlo (MCMC)
sampling~\cite{jerrum1996markov,madras2002}.  MCMC sampling is perhaps
the most popular technique for both weighted sampling and weighted
model counting.  Several MCMC algorithms like simulated annealing and
the Metropolis-Hastings algorithm have been studied extensively in the
literature~\cite{Kirkpatrick83,madras2002}.  While MCMC sampling is
guaranteed to converge to a target distribution under mild
requirements, convergence is often impractically
slow~\cite{jerrum1996markov}.  Therefore, practical MCMC
sampling-based tools use heuristics that destroy the theoretical
guarantees.
Several other heuristic techniques that provide weak one-sided bounds 
have also been proposed in the literature~\cite{Gomes06modelcounting}.

Recently, Ermon et al. proposed new hashing-based algorithms for
approximate weighted model counting and approximate weighted
sampling~\cite{ermon2013icml,ermon2013nips,ermon2013uai,ermon2014icml}.
Their algorithms use random parity constraints as pair-wise
independent hash functions to partition the set of satisfying
assignments of a CNF formula into cells.  A {\MAP} oracle is then
queried to obtain the maximum weight of an assignment in a randomly
chosen cell.  By repeating the {\MAP} queries polynomially many times
for randomly chosen cells of appropriate expected sizes, Ermon et al
showed that they can provably compute approximate weighted model
counts and also provably achieve approximate weighted sampling. The
performance of Ermon et al's algorithms depend crucially on the
ability to efficiently answer {\MAP} queries.  Complexity-wise, {\MAP}
is significantly harder than CNF satisfiability, and is known to be
${\NP}^{PP}$-complete~\cite{Park2002}.  The problem is further
compounded by the fact that the {\MAP} queries generated by Ermon et
al's algorithms have random parity constraints built into them.
Existing {\MAP}-solving techniques work efficiently when the weight
distribution of assignments is specified by a graphical model, and the
underlying graph has specific structural properties.  With random
parity constraints, these structural properties are likely to be
violated very often.  In~\cite{ermon2013uai}, it has been argued that
a {\MAP}-oracle-based weighted model-counting algorithm proposed
in~\cite{ermon2013icml} is unlikely to scale well to large problem
instances.  Since {\MAP} solving is also crucial in the weighted
sampling algorithm of~\cite{ermon2013nips}, the same criticism applies
to that algorithm as well.  Several relaxations of the
{\MAP}-oracle-based algorithm proposed in~\cite{ermon2013icml}, were
therefore discussed in~\cite{ermon2013uai}.  While these relaxations
help reduce the burden of {\MAP} solving, they also significantly
weaken the theoretical guarantees.

In later work~\cite{ermon2014icml}, Ermon et al showed how the average
size of parity constraints in their weighted model counting and
weighted sampling algorithms can be reduced using a new class of hash
functions.  This work, however, still stays within the same paradigm
as their earlier work -- i.e, it uses {\MAP}-oracles and XOR
constraints.  %
Although Ermon et al's algorithms provide a $16$-factor approximation
in theory, in actual experiments, they use relaxations and timeouts of
the {\MAP} solver to get upper and lower bounds of the optimal {\MAP}
solution.  Unfortunately, these bounds do not come with any guarantees
on the factor of approximation.  Running the {\MAP} solver to obtain
the optimal value is likely to take significantly longer, and is not
attempted in Ermon et al's work.  %

The algorithms developed in this paper are closely related to two
algorithms proposed recently by Chakraborty, Meel and
Vardi~\cite{CMV13MC,CMV-CAV13}.  The first of these~\cite{CMV13MC}
computes the approximate (unweighted) model-count of a CNF formula,
while the second algorithm~\cite{CMV-CAV13} performs near-uniform
(unweighted) sampling.  Like Ermon et al's algorithms, these
algorithms make use of parity constraints as pair-wise independent
hash functions, and can benefit from the new class of hash functions
proposed in~\cite{ermon2014icml}.  Unlike, however, Ermon et al's
algorithms, Chakraborty et al.~use a SAT solver ({\NP}-oracle)
specifically engineered to handle parity constraints efficiently.

\section{Algorithm}
\label{sec:algorithm}
We now present algorithms for approximate weighted model counting and
approximate weighted sampling, assuming a small bounded tilt and a
black-box weight function.  Recalling that the tilt concerns weights
of only satisfying assignments, our assumption about it being bounded
by a small number is reasonable in several practical situations.  For
example, when solving probabilistic inference with evidence by
reduction to weighted model counting~\cite{chavira2008probabilistic}, every
satisfying assignment of the CNF formula corresponds to an assignment
of values to variables in the underlying probabilistic graphical model
that is consistent with the evidence.  Furthermore, the weight of a
satisfying assignment is the joint probability of the corresponding
assignment of variables in the probabilistic graphical model.  A large
tilt would therefore mean existence of two assignments that are
consistent with the evidence, but one of which is overwhelmingly more
likely than the other.  In several real-world problems (see, e.g. Sec
8.3 of~\cite{diez2006canonical}), this is considered unlikely given that
numerical conditional probability values are often obtained from human
experts providing qualitative and rough quantitative data.

Our weighted model counting algorithm, called {\WeightMC}, is best
viewed as an adaptation of the {\ApproxMC} algorithm proposed by
Chakraborty, Meel and Vardi~\cite{CMV13MC} for approximate unweighted
model counting.  Similarly, our weighted sampling algorithm, called
{\WeightGen}, can be viewed as an adaptation of the the
{\UniformWitness} algorithm~\cite{CMV-CAV13}, originally proposed for
near-uniform unweighted sampling.  The key idea in both {\ApproxMC}
and {\UniformWitness} is to partition the set of satisfying
assignments into ``cells'' containing roughly equal numbers of
satisfying assignments, using a random hash function from the family
$H_{xor}(n, m, 3)$.  A random cell is then chosen and inspected to see
if the number of satisfying assignments in it is smaller than a
pre-computed threshold.  The threshold, in turn, depends on the
desired approximation factor or tolerance $\varepsilon$.  If the
chosen cell is small enough, {\UniGen} samples uniformly from the
chosen small cell to obtain a near-uniformly generated satisfying
assignment.  {\ApproxMC} multiplies the number of satisfying
assignments in the cell by a suitable scaling factor to obtain an
estimate of the model count.  {\ApproxMC} is then repeated a number of
times (depending on the desired confidence: $1-\delta$) and the
statistical median of computed counts taken to give the final
approximate model count.  For weighted model counting and sampling,
the primary modification that needs to be done to {\ApproxMC} and
{\UniGen} is that instead of requiring ``cells" to have roughly equal
numbers of satisfying assignments, we now require them to have roughly
equal weights of satisfying assignments. To ensure that all weights
lie in $[0, 1]$, we scale weights by a factor of
$\frac{1}{\mathrm{w_{max}}}$.  Unlike earlier
works~\cite{ermon2013icml,ermon2013uai}, however, we do not require a
{\MAP}-oracle to get $\mathrm{w_{max}}$; instead we estimate
$\mathrm{w_{max}}$ online without incurring any additional performance
cost.

A randomly chosen hash function from $H_{xor}(n,m,3)$ consists of $m$
XOR constraints, each of which has expected size $n/2$. Although
{\ApproxMC} and {\UniformWitness} were shown to scale for few
thousands of variables, the performance erodes rapidly after a few
thousand variables.  It has recently been showin in~\cite{msthesis}
that by using random parity constraints on the independent support of
a formula (which can be orders of magnitude smaller than the complete
support), we can significantly reduce the size of XOR constraints.  We
use this idea in our work.  For all our benchmark problems, obtaining
the independent support of CNF formulae has been easy, once we examine
the domain from which the problem originated.  %

 Both {\WeightMC} and {\WeightGen} assume access to a subroutine
 called {\BoundedWeightSAT} that takes a CNF formula $F$, a ``pivot'',
 an upper bound $r$ of the tilt and an upper bound $w_{max}$ of the
 maximum weight of a satisfying assignment in the independent support
 set $S$.  It returns a set of satisfying assignments of $F$ such that
 the total weight of the returned assignments scaled by $1/w_{max}$
 exceeds pivot.  It also updates the minimum weight of a satisfying
 assignment seen so far and returns the same.  {\BoundedWeightSAT}
 accesses a subroutine {\AddBlockClause} that takes as inputs a
 formula $F$ and a projected assignment $\sigma|_{S}$, computes a
 blocking clause for $\sigma|_{S}$, and returns the formula $F'$
 obtained by conjoining $F$ with the blocking clause thus obtained.
 Both algorithms also accept as input a positive
 real-valued parameter $r$ which is an upper bound on
 {\ratio}. Finally, the algorithms assume access to an {\NP}-oracle,
 which in particular can decide SAT.
    
\subsection{WeightMC Algorithm} 
The pseudocode for {\WeightMC} is shown in Algorithm
\ref{alg:WtMC}. The algorithm takes a CNF formula $F$, tolerance
$\varepsilon \in (0,1)$, confidence parameter $\delta \in (0,1)$,
independent support $S$, and tilt upper bound $r$, and returns an
approximate weighted model count.  {\WeightMC} invokes an auxiliary
procedure {\WeightMCCore} that computes an approximate weighted model
count by randomly partitioning the space of satisfying assignments
using hash functions from the family $H_{xor}(|S|, m, 3)$, where $S$
denotes an independent support of $F$.  After invoking {\WeightMCCore}
sufficiently many times, {\WeightMC} returns the median of the
non-$\bot$ counts returned by {\WeightMCCore}.
\begin{theorem}\label{thm:wtmcApproximation}
Given a propositional formula $F$, $\varepsilon \in(0,1)$, $\delta \in
(0,1)$, independent support $S$, and tilt bound $r$, suppose
{\WeightMC}$(F, \varepsilon, \delta, S, r)$ returns $c$.  Then
$\prob\left[{\left(1+\varepsilon\right)}^{-1}\cdot
  \weight{\satisfying{F}}) \right.$ $\left. \le c \le
  (1+\varepsilon)\cdot \weight{\satisfying{F}})\right]$ $\ge
1-\delta$.
 \end{theorem}
\begin{theorem}\label{thm:wtmcComplexity}
Given an oracle for {\SAT}, \WeightMC$(F, \varepsilon, \delta, S, r)$
runs in time polynomial in $\log_2(1/\delta), r, |F|$ and
$1/\varepsilon$ relative to the oracle.
\end{theorem}
We defer all proofs to the supplementary material for lack of space.

\subsection{WeightGen Algorithm}
The pseudocode for {\WeightGen} is presented in Algorithm
\ref{alg:WtGen}. {\WeightGen} takes in a CNF formula $F$, tolerance
$\varepsilon > 1.71$, tilt upper bound $r$, and independent support
$S$ and returns a random (approximately weighted-uniform) satisfying
assignment.
{\WeightGen} first computes $\kappa$ and $\mathrm{pivot}$ and uses
them to compute $\mathrm{hiThresh}$ and $\mathrm{loThresh}$, which
quantify the size of a ``small" cell. The easy case of the weighted
count being less than $\mathrm{hiThresh}$ is handled in lines
6--9. Otherwise, {\WeightMC} is called to estimate the weighted model
count, which is used to estimate the range of candidate values for
$m$.  The choice of parameters for {\WeightMC} is motivated by
technical reasons. The loop in 13--19 terminates when a small cell is
found and a sample is picked weighted-uniformly at random. Otherwise,
the algorithm reports a failure.
\begin{theorem}\label{thm:wtgenUniformity}
Given a CNF formula $F$, tolerance $\varepsilon > 1.71$, tilt bound
$r$, and independent support $S$, for every $y \in R_F$ we have $
\frac{\weight{y}}{(1+\varepsilon)\weight{R_F}} \le
\prob\left[{\WeightGen}(F, \varepsilon, r, X) = y\right] \le $ $
(1+\varepsilon)\frac{\weight{y}}{\weight{R_F}}.$ Also, {\WeightGen}
succeeds (i.e. does not return $\bot$) with probability at least
$0.62$.
\end{theorem} 
\begin{theorem}\label{thm:wtgenComplexity}
Given an oracle for {\SAT}, $\WeightGen(F,\varepsilon,r,S)$ runs in time polynomial in $ r, |F|$ and
$1/\varepsilon$ relative to the oracle.
\end{theorem}

\subsection{Implementation Details}
 In our implementations of {\WeightGen} and {\WeightMC},
 {\BoundedWeightSAT} is implemented using
 CryptoMiniSAT~\cite{CryptoMiniSAT}, a SAT solver that handles xor
 clauses efficiently. CryptoMiniSAT uses \emph{blocking clauses} to
 prevent already generated witnesses from being generated again.
 Since the independent support of $F$ determines every satisfying
 assignment of $F$, blocking clauses can be restricted to only
 variables in the set $S$.  We implemented this optimization in
 CryptoMiniSAT, leading to significant improvements in performance. We
 used ``random\_device" implemented in C++11 as source of
 pseudo-random numbers to make random choices in {\WeightGen} and
 {\WeightMC}.

\begin{algorithm}[h]

\scriptsize
\caption{\WeightMC$(F, \varepsilon, \delta, S, r)$}
\label{alg:WtMC}

\begin{algorithmic}[1]

\State $\mathrm{counter} \gets 0; C \gets \mathsf{emptyList}; \mathrm{w_{max}} \gets 1$; \label{line:weightmc-init-start}
\State $\mathrm{pivot} \gets 2 \times \lceil e^{3/2}\left(1 + \frac{1}{\varepsilon}\right)^2 \rceil$;
\State $t \gets \left\lceil 35\log_2 (3/\delta) \right\rceil$; \label{line:weightmc-init-end}
\Repeat \label{line:weightmc-loop-start}
	\State $(c,\mathrm{w_{max}}) \gets \WeightMCCore(F,S,\mathrm{pivot}, r,\mathrm{w_{max}})$;
	\State $\mathrm{counter} \gets \mathrm{counter}+1$;
	\If {$ c \neq \bot$}
		\State $ \mathsf{AddToList}(C,c \cdot \mathrm{w_{max}} )$;
	\EndIf
\Until { $\mathrm{counter} < t$} \label{line:weightmc-loop}
\State $\mathrm{finalCount} \gets \mathsf{FindMedian}(C)$; \label{line:weightmc-median}
\State \Return $\mathrm{finalCount}$;
\end{algorithmic}
\end{algorithm}

\begin{algorithm}
\scriptsize
\caption{\WeightMCCore$(F,S,\mathrm{pivot}, r,\mathrm{w_{max}})$}
\label{alg:WtMCCore}
\begin{algorithmic}[1]
\State $(Y,\mathrm{w_{max}}) \gets \BoundedWeightSAT(F,\mathrm{pivot}, r,\mathrm{w_{max}},S)$;
\If {$\weight{Y} / \mathrm{w_{max}} \leq \mathrm{pivot}$}
	\State \Return $\weight{Y}$;
\Else
	\State $i \gets 0$;
	\Repeat \label{line:wmccore-loop-start}
		\State $i \gets i+1$;
		\State Choose $h$ at random from $H_{xor}(|S|, i, 3)$; \label{line:WeightMCCore-choose-hash}
 		\State Choose $\alpha$ at random from $\{0, 1\}^{i}$; \label{line:weightmccore-choose-alpha}
		\State $(Y,\mathrm{w_{max}}) \gets \BoundedWeightSAT(F \wedge (h(x_1, \ldots x_{|S|}) = \alpha),\mathrm{pivot},\ratio, \mathrm{w_{max}},S)$; \label{line:WeightMCCore-call-BoundedWeightSAT}
	\Until{$ (0 < \weight{Y}/\mathrm{w_{max}} \leq \mathrm{pivot})$ or $i = n$} \label{line:wmccore-loop}
	\If {$\weight{Y}/\mathrm{w_{max}} > \mathrm{pivot}$ {\bfseries or} $\weight{Y} = 0$}
		\Return $(\bot,\mathrm{w_{max}})$;
	\Else 
		\Return $ (\frac{\weight{Y} \cdot 2^{i-1}}{\mathrm{w_{max}}},\mathrm{w_{max}})$; \label{line:wmccore-return}
	\EndIf
\EndIf
\end{algorithmic}
\end{algorithm}

\begin{algorithm}
\scriptsize
\caption{$\BoundedWeightSAT(F,\mathrm{pivot}, r,\mathrm{w_{max}},S)$}
\label{alg:bounded}
\begin{algorithmic}[1]
\State $\mathrm{w_{min}} \gets \mathrm{w_{max}}/r; \mathrm{w_{total}} \gets 0; Y = \{\} $;
\Repeat
	\State $y \gets \SolveSAT (F)$;
\If {y == \UNSAT}
	\State break;
\EndIf
\State $Y = Y \cup y$;
\State $F = \AddBlockClause(F,y|_S)$;
\State $\mathrm{w_{total}} \gets \mathrm{w_{total}} + \weight{y}$;
\State $\mathrm{w_{min}} \gets min(\mathrm{w_{min}}, \weight{y})$;
\Until {$\mathrm{w_{total}}/(\mathrm{w_{min}} \cdot r) > pivot$};
\State \Return $(Y,\mathrm{w_{min}} \cdot r)$;
\end{algorithmic}
\end{algorithm}
 {\fontsize{5pt}{5pt}\selectfont
\begin{algorithm}
\scriptsize
\caption{\WeightGen$(F, \varepsilon, r, S)$}
\label{alg:WtGen}
\begin{algorithmic}[1]
\Statex \hspace*{-0.6cm}/*Assume $\varepsilon > 1.71$ */ 
\State $\mathrm{w_{max}} \gets 1$; $\mathrm{Samples} = \{\}$;
\State $(\kappa, \mathrm{pivot}) \gets \ComputeKappaAndPivot(\varepsilon)$;
\State $\mathrm{hiThresh} \gets 1 + (1+\kappa)\mathrm{pivot}$;
\State $\mathrm{loThresh} \gets \frac{1}{1+\kappa} \mathrm{pivot}$;
\State $(Y,\mathrm{w_{max}}) \gets \BoundedWeightSAT(F, \mathrm{hiThresh}, r,\mathrm{w_{max}},S)$;
\If {$\left(\weight{Y} / \mathrm{w_{max}} \le \mathrm{hiThresh} \right)$}
 	 \State Choose $y$ weighted-uniformly at random from $Y$;
	 \State \Return $y$; \label{line:weightgen-base-return}
\Else
	 \State $(C,\mathrm{w_{max}}) \gets \WeightMC(F, 0.8, 0.2);$ \label{line:weightgen-approx-counter}
	 	\State $q \gets \lceil\log C - \log \mathrm{w_{max}}+ \log 1.8- \log \mathrm{pivot}\rceil$; \label{line:weightgen-q}
 	 	\State $i \gets q-4$;
		 \Repeat \label {line:weightgen-loop-start} 
    	 	\State $i \gets i+1$;
 			\State Choose $h$ at random from $H_{xor}(|S|, i, 3)$;
 			\State Choose $\alpha$ at random from $\{0, 1\}^{i}$; \label{line:weightgen-alpha}
			\State $(Y,\mathrm{w_{max}}) \gets \BoundedWeightSAT(F \wedge (h(x_1, \ldots x_{|S|}) = \alpha),\mathrm{hiThresh}, r, \mathrm{w_{max}},S)$; \label{line:weightgen-bwsat-2}
			\State $W \gets \weight{Y} / \mathrm{w_{max}}$
		\Until {$\left(\mathrm{loThresh} \le W \le \mathrm{hiThresh} \right)$ or ($i = q$)} \label{line:weightgen-loop}
	\If {($W > \mathrm{hiThresh}$) or ($W < \mathrm{loThresh}$)} \label{line:weightgen-failure} {\Return} $\bot$
	{\Else} Choose $y$ weighted-uniformly at random from $Y$; \Return $y$; \label{line:weightgen-last-return}
	\EndIf

\EndIf
\end{algorithmic}
\end{algorithm}
}

\begin{algorithm}
\scriptsize
\caption{$\ComputeKappaAndPivot(\varepsilon)$}
\label{alg:computeKappa}
\begin{algorithmic}[1]
\State Find $\kappa \in [0, 1)$ such that $\varepsilon = (1+\kappa)(2.36+\frac{0.51}{(1-\kappa)^2}) - 1$ ;
\State $\mathrm{pivot} \gets \lceil e^{3/2}\left(1 + \frac{1}{\kappa}\right)^2 \rceil$; \Return $(\kappa, \mathrm{pivot})$

\end{algorithmic}
\end{algorithm}

\section{Experimental Results}
\label{sec:experiments}

To evaluate the performance of {\WeightGen} and {\WeightMC}, we built 
prototype implementations and conducted an extensive set of experiments. The 
suite of benchmarks was made up of problems arising from various practical 
domains as well as problems of theoretical interest. Specifically, we
used bit-level unweighted versions of constraints arising from grid
networks, plan recognition, DQMR networks, bounded model checking of
circuits, bit-blasted versions of SMT-LIB~\cite{SMTLib} benchmarks,
and ISCAS89~\cite{brglez1989combinational} circuits  
with parity conditions on randomly chosen subsets of outputs and next-state 
variables~\cite{SangBearKautz2005,AJSC11}. While our algorithm is agnostic 
to the weight oracle, other tools that we used for comparison require the 
weight of an assignment to be the product of the weights of its literals. 
Consequently, to create weighted problems with tilt at most some bound $r$, 
we randomly selected $m = \max(15, n/100)$ of the variables and assigned 
them the weight $w$ such that $(w / (1-w))^m = r$, their negations the weight $1-w$, and all other 
literals the weight 1. Unless mentioned otherwise, our 
experiments for {\WeightMC} used $r = 3$, $\epsilon=0.8$, and $\delta=0.2$,
while our experiments for {\WeightGen} used $r =3$ and $\epsilon = 5$.

To facilitate performing multiple experiments in parallel, we used a high 
performance cluster, each experiment running on its own core.  Each node 
of the cluster had two quad-core Intel Xeon processors with 4GB of main 
memory. We used 2500 seconds as the timeout of each invocation of 
{\BoundedWeightSAT} and 20 hours as the overall timeout for 
{\WeightGen} and {\WeightMC}. If an invocation of {\BoundedWeightSAT} 
timed out in line \ref{line:WeightMCCore-call-BoundedWeightSAT} ({\WeightMC}) and line \ref{line:weightgen-bwsat-2} ({\WeightGen}), we repeated 
the execution of the corresponding loops without incrementing the variable 
$i$ (in both algorithms). With this setup, {\WeightMC} and {\WeightGen} 
were able to successfully return weighted counts and generate weighted 
random instances for formulas with close to 64,000 variables.

We compared the performance of {\WeightMC} with the {\SDD}
Package~\cite{sdd-package}, a state-of-the-art tool which can perform
exact weighted  model counting by compiling CNF formulae into Sentential 
Decision Diagrams~\cite{ChoiDarwiche13}. (We also tried to compare our tools
against {\Cachet},  {\WISH} and {\PAWS} but we have not been able to run 
these tools on our systems.)
Our results are shown in Table \ref{table:performance}, where column 1
lists the benchmarks and columns 2 and 3 give the number of variables
and clauses for each benchmark. Column 4 lists the time taken by
{\WeightMC}, while column 5 lists the time taken by {\SDD}. We also measured
the time taken by {\WeightGen} to generate samples, which we will discuss later in this section, and list it i column 6.   ``T'' and
``mem'' indicate that an experiment exceeded our imposed 20-hour and
4GB-memory limits, respectively. While {\SDD} was generally superior for 
small problems, {\WeightMC} was significantly faster for all benchmarks 
with more than 1,000 variables. 

\begin{table}
\caption{{\WeightMC}, {\SDD}, and {\WeightGen} runtimes in seconds.}
\label{table:performance}
\scriptsize
\begin{tabular}{|l|r|r|r|r|r|}
\hline
\textbf{Benchmark} & \textbf{vars} & \textbf{\#clas} & \textbf{\shortstack{Weight-\\MC}} & \textbf{SDD} &\textbf{\shortstack{Weight-\\Gen}} \\
\hline
or-50&100&266&15&0.38&0.14\\ \hline 
or-70&140&374&771&0.83&13.37\\ \hline 
s526\_3\_2&365&943&62&29.54&0.85\\ \hline 
s526a\_3\_2&366&944&81&12.16&1.1\\ \hline 
s953a\_3\_2&515&1297&11978&355.7&21.14\\ \hline 
s1238a\_7\_4&704&1926&3519&mem&19.52\\ \hline 
s1196a\_15\_7&777&2165&3087&2275&19.59\\ \hline 
Squaring9&1434&5028&34942&mem&110.37\\ \hline 
Squaring7&1628&5837&39367&mem&113.12\\ \hline 
ProcessBean&4768&14458&53746&mem&418.29\\ \hline 
LoginService2&11511&41411&322&mem&3.45\\ \hline 
Sort&12125&49611&19303&T&140.19\\ \hline 
EnqueueSeq&16466&58515&8620&mem&165.64\\ \hline 
Karatsuba&19594&82417&4962&mem&193.11\\ \hline 
TreeMax&24859&103762&34&T&2.0\\ \hline 
LLReverse&63797&257657&1496&mem&88.0\\ \hline
\end{tabular}
\end{table}

To evaluate the quality of the approximate counts returned by
{\WeightMC}, we computed exact weighted model counts using the {\SDD}
tool for a subset of our benchmarks. Figure \ref{fig:quality_comparison}
shows the counts returned by {\WeightMC}, and the exact counts from {\SDD} scaled up and down by 
$(1+\varepsilon)$. The weighted model counts are represented on the y-axis, while the x-axis represents benchmarks arranged in increasing order of counts.  We observe,
for all our experiments, that the weighted counts returned by
{\WeightMC} lie within the tolerance of the exact counts. Over all
of the benchmarks, the $L_1$ norm of the relative error was 0.036,
demonstrating that in practice {\WeightMC} is substantially more accurate 
than the theoretical guarantees provided by Theorem~\ref{thm:wtgenUniformity}. 
\begin{figure}[htb]\centering
\scriptsize
\includegraphics[width=\linewidth,clip=true,trim=0 1.1cm 0 1cm]{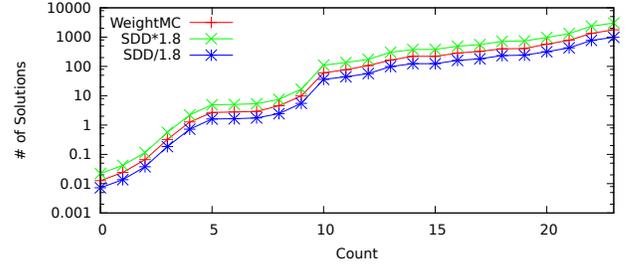}
\caption{\scriptsize Quality of counts computed by {\WeightMC}. The benchmarks are
  arranged in increasing order of weighted model counts.}
\label{fig:quality_comparison}
\end{figure}
\begin{figure}[htb]\centering
\scriptsize
\includegraphics[width=\linewidth,clip=true,trim=0 1.1cm 0 1cm]{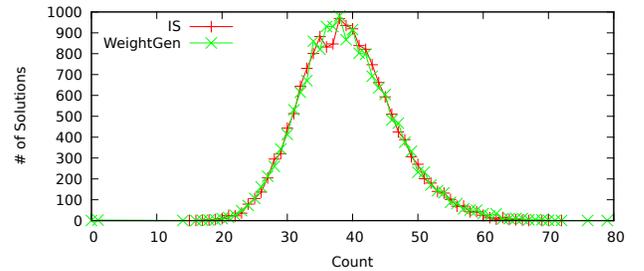}
\vspace*{-0.15in}
\caption{\scriptsize Uniformity comparison for case110}
\label{fig:wtgen-uniform}

\end{figure}

In another experiment, we studied the effect of different values of the 
tilt bound $r$ on the runtime of {\WeightMC}.  Runtime as a function $r$ is shown for several benchmarks in Figure \ref{fig:weightmc-ratio},
where times have been normalized so that at the lowest tilt ($r = 1$) 
each benchmark took one time unit. Each runtime is an average over five 
runs on the same benchmark. The theoretical linear dependence on the tilt
shown in Theorem~\ref{thm:wtmcComplexity} can be seen to roughly 
occur in practice. 
\begin{figure}[htb]\centering
\scriptsize
\includegraphics[width=\linewidth,clip=true,trim=0 0.7cm 0 1cm]{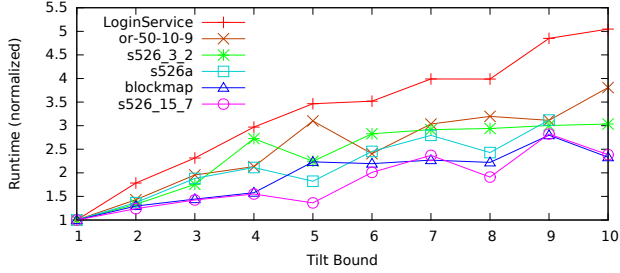}
\caption{\scriptsize Runtime of {\WeightMC} as a function of \emph{tilt bound}.}
\label{fig:weightmc-ratio}
\end{figure}

Since a probabilistic generator is likely to be invoked many times with the same formula and weights, it is useful to perform the counting on line \ref{line:weightgen-approx-counter} of {\WeightGen} only once, and reuse the result for every sample. Reflecting this, column 6 in Table
~\ref{table:performance} lists the time, averaged over a large number of runs, taken by {\WeightGen}
to generate one sample given that the weighted model count on line \ref{line:weightgen-approx-counter} has already been found. It is clear
from Table ~\ref{table:performance} that {\WeightGen} scales to
formulas with thousands of variables.  
 
To measure the accuracy of {\WeightGen}, we implemented an \emph{Ideal
Sampler}, henceforth called {\IS}, and compared the  
distributions generated by {\WeightGen} and {\IS} for a representative 
benchmark. Given a CNF formula $F$, {\IS} first generates all the 
satisfying assignments, then computes their weights and uses these to
sample the ideal distribution.  We then generated a large number $N$ 
($=6 \times 10^5$) of sample witnesses using both {\IS} and 
{\WeightGen}. In each case, the number of times various witnesses were
generated was recorded, yielding a distribution of the counts. Figure
\ref{fig:wtgen-uniform} shows the distributions generated by
{\WeightGen} and {\IS} for one of our benchmarks (case110) with 16,384
solutions. The almost perfect match between the distribution generated by
{\IS} and {\WeightGen} held also for other benchmarks.  Thus, as was the 
case for {\WeightMC}, the accuracy of {\WeightGen} is better in practice 
than that established by Theorem~\ref{thm:wtgenUniformity}.

\section{White-Box Weight Functions}
\label{sec:white-box}

As noted above, the runtime of {\WeightMC} is proportional to the tilt of 
the weight function, which means that the algorithm becomes impractical 
when the tilt is large. If the assignment weights are given by a known 
polynomial-time-computable function instead of an oracle, we can do better.
We abuse notation slightly and 
denote this weight function by $\weight{X}$, where $X$ is the set
of support variables of the Boolean formula $F$.
The essential idea is to partition the set of satisfying assignments into 
regions within which the tilt is small. Defining
$
\satisfying{F}(a,b) = \{ \sigma \in \satisfying{F} | 
a < \weight{\sigma} \le b \} ,
$
we have $\weight{\satisfying{F}} = \weight{\satisfying{F}(\wmin, \wmax)}$. 
If we use a partition of the form 
$\satisfying{F}(\wmin, \wmax) = 
\satisfying{F}(\wmax / 2, \wmax) \cup 
\satisfying{F}(\wmax / 4, \wmax / 2) \cup \dots 
\cup \satisfying{F}(\wmax / 2^N, \wmax / 2^{N-1})$, 
where $\wmax / 2^N \le \wmin$, then in each partition region the tilt is 
at most 2.  Note that we do not need to know the actual values of $\wmin$ 
and $\wmax$: any bounds $L$ and $H$ such that 
$0 < L \le \wmin$ and $\wmax \le H$ will do (although if the bounds are too
loose, we may partition $\satisfying{F}$ into more regions than necessary).
If assignment weights are poly-time computable, we can add to $F$ a 
constraint that eliminates all assignments not in a particular region. 
So we can run {\WeightMC} on each region in turn, passing 2 as the upper 
bound on the tilt, and sum the results to get $\weight{\satisfying{F}}$. 
This idea is implemented in {\PartitionedWeightMC} 
(Algorithm \ref{alg:PartitionedWeightMC}).

\begin{algorithm}[h]
\caption{\PartitionedWeightMC$(F, \varepsilon, \delta, S, L, H)$}
\label{alg:PartitionedWeightMC}
\begin{algorithmic}[1]
\State $N \leftarrow \ceil{\log_2 H / L} + 1$; $\delta' \leftarrow \delta / N$; $c \leftarrow 0$
\ForAll{$1 \le m \le N$}
	\State $G \leftarrow F \land (H / 2^m < \weight{X} \le H / 2^{m-1})$ 
	\State $d \leftarrow \WeightMC(G, \varepsilon, \delta', S, 2)$
	\If{$(d = \bot)$}
		\Return $\bot$ %
	\EndIf
	\State $c \leftarrow c + d$
\EndFor
\State \Return $c$
\end{algorithmic}
\end{algorithm}

The correctness and runtime of {\PartitionedWeightMC} are established by 
the following theorems, whose proof is deferred to 
Appendix.
\begin{theorem}
\label{theorem:partitioned-approx}
If $\PartitionedWeightMC(F, \varepsilon, \delta, S, L, H)$ returns $c$ 
(and all arguments are in the required ranges), then
\[
\Pr \left[ c \ne \bot \land (1 + \varepsilon)^{-1} \weight{\satisfying{F}} \le c \le (1 + \varepsilon) \weight{\satisfying{F})} \right] \ge 1 - \delta .
\]
\end{theorem}

\begin{theorem}
\label{theorem:partitioned-runtime}
With access to an {\NP} oracle, the runtime of 
$\PartitionedWeightMC(F, \varepsilon, \delta, S, L, H)$ is polynomial in 
$|F|$, $1/\varepsilon$, $\log (1/\delta)$, and $\log r = \log (H / L)$.
\end{theorem}

The reduction of the runtime's dependence on the tilt bound $r$ from linear
to logarithmic can be a substantial saving. If the assignment weights are 
products of literal weights, as is the case in many applications, the best 
\emph{a priori} bound on the tilt {\ratio} given only the literal weights 
is exponential in $n$. Thus, unless the structure of the problem allows a 
better bound on {\ratio} to be used, {\WeightMC} will not be practical. 
In this situation {\PartitionedWeightMC} can be used to maintain polynomial
runtime.

When implementing {\PartitionedWeightMC} in practice
the handling of the weight constraint
$H / 2^m < \weight{X} \le H / 2^{m-1}$ is critical to
efficiency. If assignment weights are sums of literal weights, or
equivalently products of literal weights (we just take logarithms),
then the weight constraint is a pseudo-Boolean constraint. In this
case we may replace the {\SAT}-solver used by {\WeightMC} with a
pseudo-Poolean satisfiability (PBS) solver. While a number of
PBS-solvers exist~\cite{pbscomp-2012}, none have the specialized
handling of XOR clauses that is critical in making {\WeightMC}
practical. The design of such solvers is a clear direction for future
work. We also note that the choice of 2 as the tilt bound for each region 
is arbitrary, and the value may be adjusted depending on the application: 
larger values will decrease the number of regions, but increase the 
difficulty of counting within each region. Finally, note that the same partitioning idea can be used to reduce
{\WeightGen}'s dependence on $r$ to be logarithmic.

\section{Conclusion}
\label{sec:conclusion}
In this paper, we considered approximate approaches to the twin problems of
distribution-aware sampling and weighted model counting for SAT. For
approximation techniques that provide strong theoretical two-way
bounds, a major limitation is the reliance on potentially-expensive
maximum a posteriori (MAP) queries. We showed how to remove this reliance
on MAP queries, while retaining strong theoretical guarantees. First,
we provided model counting and sampling algorithms that work with a
black-box model of giving weights to assignments, requiring access
only to an $\NP$-oracle, which is efficient for small tilt values.
Experimental results demonstrate the effectiveness of this approach 
in practice. Second, we provide an alternative approach that promises to be
efficient for tilt value, requiring, however, a white-box model of weighting 
and access to a pseudo-Boolean solver. As a next step, we plan to empirically 
evaluate this latter approach using pseudo-Boolean solvers designed to
handle parity constraints efficiently.

\clearpage
\bibliography{Report}

\begin{thebibliography}{}

\bibitem[\protect\citeauthoryear{Bacchus, Dalmao, and
  Pitassi}{2003}]{Bacchus2003}
Bacchus, F.; Dalmao, S.; and Pitassi, T.
\newblock 2003.
\newblock Algorithms and complexity results for \#{SAT} and {B}ayesian
  inference.
\newblock In {\em Proc. of FOCS},  340--351.

\bibitem[\protect\citeauthoryear{Bellare, Goldreich, and
  Petrank}{1998}]{Bellare98uniformgeneration}
Bellare, M.; Goldreich, O.; and Petrank, E.
\newblock 1998.
\newblock Uniform generation of {NP}-witnesses using an {NP}-oracle.
\newblock {\em Information and Computation} 163(2):510--526.

\bibitem[\protect\citeauthoryear{Brglez, Bryan, and
  Kozminski}{1989}]{brglez1989combinational}
Brglez, F.; Bryan, D.; and Kozminski, K.
\newblock 1989.
\newblock Combinational profiles of sequential benchmark circuits.
\newblock In {\em ISCAS}.

\bibitem[\protect\citeauthoryear{Chakraborty, Meel, and Vardi}{2013a}]{CMV13MC}
Chakraborty, S.; Meel, K.~S.; and Vardi, M.~Y.
\newblock 2013a.
\newblock A scalable approximate model counter.
\newblock In {\em Proc. of CP},  200--216.

\bibitem[\protect\citeauthoryear{Chakraborty, Meel, and
  Vardi}{2013b}]{CMV-CAV13}
Chakraborty, S.; Meel, K.; and Vardi, M.
\newblock 2013b.
\newblock A scalable and nearly uniform generator of {SAT} witnesses.
\newblock In {\em Proc. of CAV}.

\bibitem[\protect\citeauthoryear{Chakraborty, Meel, and Vardi}{2014}]{msthesis}
Chakraborty, S.; Meel, K.; and Vardi, M.
\newblock 2014.
\newblock Balancing scalability and uniformity in sat witness generator.
\newblock In {\em To Appear in DAC}.

\bibitem[\protect\citeauthoryear{Chavira and
  Darwiche}{2008}]{chavira2008probabilistic}
Chavira, M., and Darwiche, A.
\newblock 2008.
\newblock On probabilistic inference by weighted model counting.
\newblock {\em Artificial Intelligence} 172(6):772--799.

\bibitem[\protect\citeauthoryear{Choi and Darwiche}{2013}]{ChoiDarwiche13}
Choi, A., and Darwiche, A.
\newblock 2013.
\newblock Dynamic minimization of sentential decision diagrams.
\newblock In {\em Proc. of AAAI},  187--194.

\bibitem[\protect\citeauthoryear{Cry}{}]{CryptoMiniSAT}
{CryptoMiniSAT}.
\newblock \url{http://www.msoos.org/cryptominisat2/}.

\bibitem[\protect\citeauthoryear{Darwiche}{2004}]{darwiche2004new}
Darwiche, A.
\newblock 2004.
\newblock New advances in compiling {CNF} to decomposable negation normal form.
\newblock In {\em Proc. of ECAI},  328--332.
\newblock Citeseer.

\bibitem[\protect\citeauthoryear{D{\i}ez and
  Druzdzel}{2006}]{diez2006canonical}
D{\i}ez, F.~J., and Druzdzel, M.~J.
\newblock 2006.
\newblock Canonical probabilistic models for knowledge engineering.
\newblock Technical report, Technical Report CISIAD-06-01, UNED, Madrid, Spain.

\bibitem[\protect\citeauthoryear{Domshlak and Hoffmann}{2007}]{domshlak2007}
Domshlak, C., and Hoffmann, J.
\newblock 2007.
\newblock Probabilistic planning via heuristic forward search and weighted
  model counting.
\newblock {\em Journal of Artificial Intelligence Research} 30(1):565--620.

\bibitem[\protect\citeauthoryear{Ermon \bgroup et al\mbox.\egroup
  }{2013a}]{ermon2013icml}
Ermon, S.; Gomes, C.; Sabharwal, A.; and Selman, B.
\newblock 2013a.
\newblock Taming the curse of dimensionality: Discrete integration by hashing
  and optimization.
\newblock In {\em Proc. of ICML},  334--342.

\bibitem[\protect\citeauthoryear{Ermon \bgroup et al\mbox.\egroup
  }{2013b}]{ermon2013nips}
Ermon, S.; Gomes, C.~P.; Sabharwal, A.; and Selman, B.
\newblock 2013b.
\newblock Embed and project: Discrete sampling with universal hashing.
\newblock In {\em Proc of NIPS},  2085--2093.

\bibitem[\protect\citeauthoryear{Ermon \bgroup et al\mbox.\egroup
  }{2013c}]{ermon2013uai}
Ermon, S.; Gomes, C.~P.; Sabharwal, A.; and Selman, B.
\newblock 2013c.
\newblock Optimization with parity constraints: From binary codes to discrete
  integration.

\bibitem[\protect\citeauthoryear{Ermon \bgroup et al\mbox.\egroup
  }{2014}]{ermon2014icml}
Ermon, S.; Gomes, C.; Sabharwal, A.; and Selman, B.
\newblock 2014.
\newblock Low-density parity constraints for hashing-based discrete
  integration.
\newblock In {\em Proc. of ICML},  271--279.

\bibitem[\protect\citeauthoryear{Gogate and Dechter}{2011}]{GogDech2011}
Gogate, V., and Dechter, R.
\newblock 2011.
\newblock Samplesearch: Importance sampling in presence of determinism.
\newblock {\em Artificial Intelligence} 175(2):694--729.

\bibitem[\protect\citeauthoryear{Gomes, Sabharwal, and
  Selman}{2006}]{Gomes06modelcounting}
Gomes, C.; Sabharwal, A.; and Selman, B.
\newblock 2006.
\newblock Model counting: A new strategy for obtaining good bounds.
\newblock In {\em Proc. of AAAI},  54--61.

\bibitem[\protect\citeauthoryear{Gomes, Sabharwal, and
  Selman}{2007}]{Gomes-Sampling}
Gomes, C.; Sabharwal, A.; and Selman, B.
\newblock 2007.
\newblock Near uniform sampling of combinatorial spaces using {XOR}
  constraints.
\newblock In {\em Proc. of NIPS},  670--676.

\bibitem[\protect\citeauthoryear{Jerrum and Sinclair}{1996}]{jerrum1996markov}
Jerrum, M., and Sinclair, A.
\newblock 1996.
\newblock The {Markov} chain {Monte Carlo} method: an approach to approximate
  counting and integration.
\newblock {\em Approximation algorithms for NP-hard problems}  482--520.

\bibitem[\protect\citeauthoryear{Jerrum, Valiant, and Vazirani}{1986}]{Jerr}
Jerrum, M.; Valiant, L.; and Vazirani, V.
\newblock 1986.
\newblock Random generation of combinatorial structures from a uniform
  distribution.
\newblock {\em TCS} 43(2-3):169--188.

\bibitem[\protect\citeauthoryear{John and Chakraborty}{2011}]{AJSC11}
John, A., and Chakraborty, S.
\newblock 2011.
\newblock A quantifier elimination algorithm for linear modular equations and
  disequations.
\newblock In {\em Proc. of CAV},  486--503.

\bibitem[\protect\citeauthoryear{Karp, Luby, and Madras}{1989}]{KarpLuby1989}
Karp, R.; Luby, M.; and Madras, N.
\newblock 1989.
\newblock {M}onte-{C}arlo approximation algorithms for enumeration problems.
\newblock {\em Journal of Algorithms} 10(3):429--448.

\bibitem[\protect\citeauthoryear{Kirkpatrick, Gelatt, and
  Vecchi}{1983}]{Kirkpatrick83}
Kirkpatrick, S.; Gelatt, C.~D.; and Vecchi, M.~P.
\newblock 1983.
\newblock Optimization by simulated annealing.
\newblock {\em Science} 220(4598):671--680.

\bibitem[\protect\citeauthoryear{Madras and
  Piccioni}{1999}]{madras1999importance}
Madras, N., and Piccioni, M.
\newblock 1999.
\newblock Importance sampling for families of distributions.
\newblock {\em Annals of applied probability}  1202--1225.

\bibitem[\protect\citeauthoryear{Madras}{2002}]{madras2002}
Madras, N.
\newblock 2002.
\newblock Lectures on monte carlo methods, fields institute monographs 16.
\newblock {\em AMS}.

\bibitem[\protect\citeauthoryear{Manquinho and Roussel}{2012}]{pbscomp-2012}
Manquinho, V., and Roussel, O.
\newblock 2012.
\newblock Seventh pseudo-boolean competition.
\newblock \url{http://www.cril.univ-artois.fr/PB12/}.

\bibitem[\protect\citeauthoryear{Naveh \bgroup et al\mbox.\egroup
  }{2006}]{NavRim07}
Naveh, Y.; Rimon, M.; Jaeger, I.; Katz, Y.; Vinov, M.; Marcus, E.; and Shurek,
  G.
\newblock 2006.
\newblock Constraint-based random stimuli generation for hardware verification.
\newblock In {\em Proc of IAAI},  1720--1727.

\bibitem[\protect\citeauthoryear{Park}{2002}]{Park2002}
Park, J.~D.
\newblock 2002.
\newblock Map complexity results and approximation methods.
\newblock In {\em Proceedings of UAI},  388--396.

\bibitem[\protect\citeauthoryear{Roth}{1996}]{Roth1996}
Roth, D.
\newblock 1996.
\newblock On the hardness of approximate reasoning.
\newblock {\em Artificial Intelligence} 82(1):273--302.

\bibitem[\protect\citeauthoryear{Sang, Bearne, and
  Kautz}{2005}]{SangBearKautz2005}
Sang, T.; Bearne, P.; and Kautz, H.
\newblock 2005.
\newblock Performing bayesian inference by weighted model counting.
\newblock In {\em Prof. of AAAI},  475--481.

\bibitem[\protect\citeauthoryear{Sang \bgroup et al\mbox.\egroup
  }{2004}]{Sang04combiningcomponent}
Sang, T.; Bacchus, F.; Beame, P.; Kautz, H.; and Pitassi, T.
\newblock 2004.
\newblock Combining component caching and clause learning for effective model
  counting.
\newblock In {\em Proc. of SAT}.

\bibitem[\protect\citeauthoryear{sdd}{}]{sdd-package}
The {SDD} package.
\newblock \url{http://reasoning.cs.ucla.edu/sdd/}.

\bibitem[\protect\citeauthoryear{SMT}{}]{SMTLib}
{SMTLib}.
\newblock \url{http://goedel.cs.uiowa.edu/smtlib/}.

\bibitem[\protect\citeauthoryear{Valiant}{1979}]{valiant1979complexity}
Valiant, L.
\newblock 1979.
\newblock The complexity of enumeration and reliability problems.
\newblock {\em SIAM Journal on Computing} 8(3):410--421.

\bibitem[\protect\citeauthoryear{Wainwright and
  Jordan}{2008}]{wainwright2008graphical}
Wainwright, M.~J., and Jordan, M.~I.
\newblock 2008.
\newblock Graphical models, exponential families, and variational inference.
\newblock {\em Foundations and Trends in Machine Learning}  1--305.

\bibitem[\protect\citeauthoryear{Xue, Choi, and Darwiche}{2012}]{xue2012basing}
Xue, Y.; Choi, A.; and Darwiche, A.
\newblock 2012.
\newblock Basing decisions on sentences in decision diagrams.
\newblock In {\em AAAI}.

\end{thebibliography}
\bibliographystyle{aaai}

\clearpage
\appendix

\setcounter{theorem}{0}
\setcounter{lemma}{0}
\section*{APPENDIX}
Using notation introduced in Section~\ref{sec:prelims}, let $\weight{y}$ denote the weight of solution $y$ 
and $\satisfying{F}$ denote the set of witnesses of the Boolean formula $F$. We denote the weight of the set $\satisfying{F}$ by $\weight{\satisfying{F}}$. For brevity, we write $\scaledweight{y}$ for the expression $\weight{y} / \mathrm{w_{max}}$ (where $\mathrm{w_{max}}$ is the variable appearing in several of our algorithms).

Recall that {\WeightMC} is a probabilistic algorithm that takes as
inputs a Boolean CNF formula $F$, a tolerance $\varepsilon$, confidence parameter $\delta$, a
subset $S$ of the support of $F$, and an upper bound $r$ on the ratio {\ratio}. We extend the results in ~\cite{msthesis} and show that if $X$ is the support of $F$, and if $S \subseteq X$ is an independent support of
$F$, then {\WeightMC}($F$, $\varepsilon$, $\delta$, $S$, $r$)
behaves \emph{identically} (in a probabilistic sense) to
{\WeightMC}($F$, $\varepsilon$, $\delta$, $X$, $r$).  Once this is established, the
remainder of the proof proceeds by making the simplifying assumption
$S = X$. The proofs of Lemmas \ref{lemma:ind-decomp} and \ref{lemma:XequivS} extend the earlier results by ~\cite{msthesis} for unweighted sample space.

Clearly, the above claim holds trivially if $X = S$.  Therefore, we
focus only on the case when $S \subsetneq X$.  For notational
convenience, we assume $X = \{x_1, \ldots x_n\}$, $0 \le k < n$, $S
= \{x_1, \ldots x_k\}$ and $D = \{x_{k+1}, \ldots x_n\}$ in all the
statements and proofs in this section.  We also use $\vec{X}$ to
denote the vector $(x_1, \ldots x_n)$, and similarly for $\vec{S}$ and
$\vec{D}$.
\begin{lemma}
\label{lemma:ind-decomp}
Let $F(\vec{X})$ be a Boolean function, and $S$ an independent support of $F$.  Then there exist Boolean functions
$g_0, g_1, \ldots g_{n-k}$, each with support $S$ such that
\[ 
F(\vec{X}) \leftrightarrow \left(g_0(\vec{S}) \wedge \bigwedge_{j=1}^{n-k}(x_{k+j} \leftrightarrow g_j(\vec{S}))\right)
\]
\end{lemma}
\begin{proof}
Since $S$ is an independent support of $F$, the set $D = X \setminus
S$ is a dependent support of $F$.  From the definition of a dependent
support, there exist Boolean functions $g_1, \ldots g_k$, each with
support $S$, such that $F(\vec{X}) \rightarrow \bigwedge_{j=1}^{n-k}
(x_{k+j} \leftrightarrow g_j(\vec{S}))$.

Let $g_0(\vec{S})$ be the characteristic function of the projection
of $\satisfying{F}$ on $S$.  More formally, $g_0(\vec{S}) \equiv
\bigvee_{(x_{k+1}, \ldots x_n) \in \{0, 1\}^{n-k}} F(\vec{X})$.
It follows that $F(\vec{X}) \rightarrow g_0(\vec{S})$.  Combining this
with the result from the previous paragraph, we get the implication
$F(\vec{X}) \;\rightarrow\;$ $\left(g_0(\vec{S}) \wedge \bigwedge_{j=1}^{n-k}(x_{k+j}\leftrightarrow
g_j(\vec{S}))\right)$

From the definition of $g_0(\vec{S})$ given above, we have
$g_0(\vec{S}) \rightarrow F(\vec{S}, x_{k+1}, \ldots x_n)$, for some
values of $x_{k+1}, \ldots x_n$.  However, we also know that
$F(\vec{X}) \rightarrow \bigwedge_{j=1}^{n-k} (x_{k+j} \leftrightarrow
g_j(\vec{S}))$.  It follows that
$\left(g_0(\vec{S}) \wedge \bigwedge_{j=1}^{n-k}
(x_{k+j} \leftrightarrow g_j(\vec{S}))\right) \rightarrow F(\vec{X})$.
\end{proof}

Referring to the pseudocode of {\WeightMC} in
Section~\ref{sec:algorithm}, we observe that the only steps that
depend directly on $S$ are those in line \ref{line:WeightMCCore-choose-hash}, where $h$ is chosen
randomly from $H_{xor}(|S|, i, 3)$, and line \ref{line:WeightMCCore-call-BoundedWeightSAT}, where the set $Y$ is computed by calling {\BoundedWeightSAT}($F \wedge (h(x_1, \ldots x_{|S|})
= \alpha), pivot, r, \mathrm{w_{max}}$).  Since all subsequent steps of the
algorithm depend only on $Y$, it suffices to show that if $S$ is an
independent support of $F$, the probability distribution of $Y$
obtained at line \ref{line:WeightMCCore-call-BoundedWeightSAT} is \emph{identical} to what we would obtain if $S$ was set equal to the entire support $X$.

The following lemma formalizes the above statement.  As before, we
assume $X = \{x_1, \ldots x_n\}$ and $S = \{x_1, \ldots x_k\}$.
\begin{lemma}
\label{lemma:XequivS}
Let $S$ be an independent support of $F(\vec{X})$.  Let $h$ and
$h'$ be hash functions chosen uniformly at random from $H_{xor}(k, i,
3)$ and $H_{xor}(n, i, 3)$, respectively.  Let $\alpha$ and
$\alpha'$ be tuples chosen uniformly at random from $\{0, 1\}^i$.
Then, for every $Y \in \{0, 1\}^n$, $pivot > 0$, $r \ge 1$, and $\mathrm{w_{max}} \ge 1$, we have\
\begin{align*}
&\prob\left[{\BoundedWeightSAT}\left(F(\vec{X}) \wedge (h(\vec{S}) = \alpha), pivot, r, \mathrm{w_{max}} \right) = Y\right] \\
& = \prob\left[{\BoundedWeightSAT}\left(F(\vec{X}) \wedge (h'(\vec{X}) = \alpha'), pivot, r, \mathrm{w_{max}} \right) = Y\right]
\end{align*}
\end{lemma}
\begin{proof}
Since $h'$ is chosen uniformly at random from $H_{xor}(n, i, 3)$,
recalling the definition of the latter we have
$F(\vec{X}) \wedge (h'(\vec{X}) = \alpha')$ $\equiv$
$F(\vec{X}) \wedge \bigwedge_{l=1}^i \left((a_{l,0} \oplus \bigoplus_{j=1}^{n}
a_{l,j}\cdot x[j]) \leftrightarrow \alpha'[l]\right)$, where the coefficients
$a_{l,j}$ are chosen i.i.d. uniformly from $\{0, 1\}$.

Since $S$ is an independent support of $F$, from
Lemma~\ref{lemma:ind-decomp}, there exist Boolean functions
$g_1, \ldots g_{n-k}$, each with support $S$, such that
$F(\vec{X}) \rightarrow \bigwedge_{j=1}^{n-k} (x_{k+j} \leftrightarrow
g_j(\vec{S}))$.  Therefore, $F(\vec{X}) \wedge (h'(\vec{X})
= \alpha')$ holds iff $F(\vec{X}) \wedge \bigwedge_{l=1}^i \left((a_{l,0} \oplus \bigoplus_{j=1}^{k}
a_{l,j}\cdot x[j] \oplus B) \leftrightarrow \alpha'[l]\right)$ does, where
$B \equiv \bigoplus_{j=k+1}^{n} a_{l,j}\cdot g_{j-k}(\vec{S})$.
Rearranging terms, we get $F(\vec{X}) \wedge
\bigwedge_{l=1}^i \left((a_{l,0} \oplus  \bigoplus_{j=1}^{k} a_{l,j}\cdot
 x[j]) \leftrightarrow (\alpha'[l] \oplus B)\right)$. Now since $\alpha'$ is chosen uniformly at random from $\{0,1\}^i$ and since $B$ is independent of $\alpha'$, we have that
$\alpha'[l] \oplus B$ is a random binary variable with equal probability of being $0$ and $1$. So $\alpha'[l] \oplus B$ has the same distribution as $\alpha[l]$, and the result follows.
\end{proof}

Lemma~\ref{lemma:XequivS} allows us to continue with the remainder of
the proof assuming $S = X$.  It has already been shown
in~\cite{Gomes-Sampling} that $H_{xor}(n, m, 3)$ is a $3$-independent
family of hash functions. We use this fact in a key way in the
remainder of our analysis.  The following result about
Chernoff-Hoeffding bounds, proved in~\cite{SKV13MC}, plays an
important role in our discussion.

\begin{lemma}\label{theorem:chernoff-hoeffding}
Let $\Gamma$ be the sum of $r$-wise independent random variables, each
of which is confined to the interval $[0, 1]$, and suppose
$\expect[\Gamma] = \mu$.  For $0 < \beta \le 1$, if $2 \le r \le 3$ and $\mu \ge \lceil e^{3/2}\beta^2 \rceil$ , then $\prob\left[\,|\Gamma - \mu| \ge
  \beta\mu\,\right] \le e^{-3/2}$.
\end{lemma}

\section{Analysis of WeightMC} \label{appendix:weightMC}
In this section we denote the quantity $\log_2 \scaledweight{R_F} - \log_2 pivot + 1$ by $m$. For simplicity of exposition, we assume henceforth that $m$ is an integer. A more careful analysis removes this restriction with only a constant factor scaling of the probabilities.

\begin{lemma}\label{lm:probProof}
 Let algorithm {\WeightMCCore}, when invoked from {\WeightMC}, return
 $c$ with $i$ being the final value of the loop counter in {\WeightMCCore}.  
Then $\prob\left[(1 + \varepsilon)^{-1}\cdot \scaledweight{\satisfying{F}} \le c \right.$ $\left. \le (1 +
   \varepsilon)\cdot \scaledweight{\satisfying{F}} \Bigm|  c \neq \bot \land i \leq m \right]$ $\ge 1 - e^{-3/2}$.
\end{lemma}
\begin{proof}
Referring to the pseudocode of {\WeightMCCore}, the lemma is trivially
satisfied if $\scaledweight{\satisfying{F}} \le \mathit{pivot}$.  Therefore, the only
non-trivial case to consider is when $\scaledweight{\satisfying{F}} > \mathit{pivot}$ and
{\WeightMCCore} returns from line \ref{line:wmccore-return}.  In this
case, the count returned is $2^{i} \cdot \scaledweight{\satisfying{F,h,\alpha}}$, where $\alpha, i$ and $h$
denote (with abuse of notation) the values of the corresponding
variables and hash functions in the final iteration of the
repeat-until loop in lines \ref{line:wmccore-loop-start}--\ref{line:wmccore-loop} of the pseudocode.  
From the pseudocode of {\WeightMCCore}, we know that
$pivot= \lceil e^{3/2}(1+1/\varepsilon)^{2}\rceil$.
 The lemma is now proved by showing that for every $i$ in $\{0, \ldots m\}$, $h \in H(n, i, 3)$, and $\alpha \in
\{0,1\}^{i}$, we have $\prob\left[(1 + \varepsilon)^{-1}\cdot \scaledweight{\satisfying{F}}
  \right.$ $\left. \le 2^{i}\scaledweight{\satisfying{F,h,\alpha}}\right.$ $\left.\le (1 +
  \varepsilon)\cdot \scaledweight{\satisfying{F}}\right]$ $\ge 1 - e^{-3/2}$.

For every $y \in \{0, 1\}^n$ and $\alpha \in \{0,
1\}^{i}$, define an indicator variable $\gamma_{y, \alpha}$ as
follows: $\gamma_{y, \alpha} = \scaledweight{y}$ if $h(y) = \alpha$, and
$\gamma_{y,\alpha} = 0$ otherwise.  Let us fix $\alpha$ and $y$ and
choose $h$ uniformly at random from $H(n, i, 3)$.  The random choice
of $h$ induces a probability distribution on $\gamma_{y, \alpha}$
such that $\prob\left[\gamma_{y, \alpha} = \scaledweight{y}\right] = \prob\left[h(y)
  = \alpha\right] = 2^{-i}$, and
$\expect\left[\gamma_{y,\alpha}\right] = \scaledweight{y} \prob\left[\gamma_{y, \alpha}
  = \scaledweight{y} \right] =2^{-i} \scaledweight{y}$.  In addition, the $3$-wise independence of
hash functions chosen from $H(n, i, 3)$ implies that for every
distinct $y_a, y_b, y_c \in \satisfying{F}$, the random variables $\gamma_{y_a,
  \alpha}$, $\gamma_{y_b, \alpha}$ and $\gamma_{y_c, \alpha}$ are
$3$-wise independent.

Let $\Gamma_\alpha = \sum_{y \in \satisfying{F}} \gamma_{y, \alpha}$ and
$\mu_\alpha = \expect\left[\Gamma_\alpha\right]$.  Clearly,
$\Gamma_\alpha = \scaledweight{\satisfying{F, h, \alpha}}$ and $\mu_\alpha = \sum_{y \in
  \satisfying{F}} \expect\left[\gamma_{y, \alpha}\right] = 2^{-i}\scaledweight{\satisfying{F}}$.
Therefore, using Lemma \ref{theorem:chernoff-hoeffding} with $\beta = \varepsilon / (1 + \varepsilon)$, we have $\prob\left[\scaledweight{\satisfying{F}} \left(1-\frac{\varepsilon}{1+\varepsilon}\right) \leq 2^{i}\scaledweight{\satisfying{F,h,\alpha}} \leq
  (1+\frac{\varepsilon}{1+\varepsilon})\scaledweight{\satisfying{F}}\right] $ $\ge 1- e^{-3/2}$.
 Simplifying and noting that $\frac{\varepsilon}{1+\varepsilon} <
\varepsilon$ for all $\varepsilon > 0$, we obtain
$\prob\left[(1+\varepsilon)^{-1}\cdot \scaledweight{\satisfying{F}} \leq 2^{i}\scaledweight{\satisfying{F,h,\alpha}}\leq (1+ \varepsilon)\cdot \scaledweight{\satisfying{F}} \right] $ $ \ge 1- e^{-3/2}$.
\end{proof}

\begin{lemma}\label{lm:nonbotProb}
Given $\scaledweight{\satisfying{F}} > \mathit{pivot}$, the probability that an invocation of
{\WeightMCCore} from {\WeightMC} returns non-$\bot$ with $i \le m$, is at least $1-e^{-3/2}$.
\end{lemma}
\begin{proof}

Let $p_i~(0 \le i \le n)$ denote the conditional probability that
{\WeightMCCore} terminates in iteration $i$ of
the repeat-until loop (lines \ref{line:wmccore-loop-start}--\ref{line:wmccore-loop} of the pseudocode) with $0 <
\scaledweight{\satisfying{F,h,\alpha}} \le \mathit{pivot}$, given $\scaledweight{\satisfying{F}} > \mathit{pivot}$.
Since the choice of $h$ and $\alpha$ in each iteration of the loop are
independent of those in previous iterations, the conditional
probability that {\WeightMCCore} returns
non-$\bot$ with $i \le m$, given $\scaledweight{\satisfying{F}} >
\mathit{pivot}$, is $p_0 + (1-p_0)p_{1}$ $+ \cdots +
(1-p_0)(1-p_{1})\cdots(1-p_{m-1})p_{m}$. Let us denote this sum
by $P$.  Thus, $P = p_0 + \sum_{i=1}^{m} \prod_{k=0}^{i-1}
(1-p_k)p_i$ $\,\ge\, \left(p_0 + \sum_{i=1}^{m-1}
\prod_{k=0}^{i-1} (1-p_k)p_i\right)p_{m}$ $+$ $\prod_{s=0}^{m-1}
(1-p_s)p_{m}$ $= p_{m}$.  The lemma is now proved by showing that $p_{m} \ge
1-e^{-3/2}$.

It was shown in the proof of Lemma~\ref{lm:probProof} that
$\prob\left[(1+\varepsilon)^{-1}\cdot \scaledweight{\satisfying{F}} \leq
  2^{i}\scaledweight{\satisfying{F,h,\alpha}} \leq (1+ \varepsilon)\cdot
  \scaledweight{\satisfying{F}} \right] $ $\ge 1- e^{-3/2}$ for every
   $i \in \{0,\ldots,m\}$, $h \in H(n, i, 3)$ and $\alpha \in \{0,1\}^{i}$.
Substituting $m$ for $i$, re-arranging terms and
noting that the definition of $m$ implies $2^{-m}\scaledweight{\satisfying{F}} =
\mathit{pivot} / 2$, we get
$\prob\left[(1+\varepsilon)^{-1}(\mathit{pivot} / 2)\right.$
  $\left. \leq \scaledweight{\satisfying{F,h,\alpha}}\right.$ $\left.\leq (1+
  \varepsilon)(\mathit{pivot} / 2) \right] \ge 1- e^{-3/2}$.  Since $0 <
\varepsilon \le 1$ and $\mathit{pivot} > 4$, it follows that
$\prob\left[0 < \scaledweight{\satisfying{F,h,\alpha}} \le \mathit{pivot}\right]$ $\ge$
$1-e^{-3/2}$.  Hence, $p_{m} \ge 1-e^{-3/2}$.
\end{proof}

\begin{lemma}\label{thm:almost-approx}
Let an invocation of {\WeightMCCore} from {\WeightMC} return $c$. Then
$\prob\left[c \neq \bot \land (1 + \varepsilon)^{-1}\cdot \weight{\satisfying{F}} \le c \cdot \mathrm{w_{max}} \le \right.$ $\left. (1 + \varepsilon) \cdot \weight{\satisfying{F}}\right]$ $\ge (1 - e^{-3/2})^2 > 0.6$.
\end{lemma}
\begin{proof}
It is easy to see that the required probability is at least as large as $\prob\left[c \neq \bot \land i \leq m \land (1 + \varepsilon)^{-1} \weight{\satisfying{F}} \le c \cdot \mathrm{w_{max}} \right.$ $\left. \le (1 + \varepsilon)\cdot \weight{\satisfying{F}}\right]$. Dividing by $\mathrm{w_{max}}$ and applying Lemmas~\ref{lm:probProof} and \ref{lm:nonbotProb}, this probability is $\ge (1 - e^{-3/2})^2$.
\end{proof}

We now turn to proving that the confidence can be raised to at least
$1-\delta$ for $\delta \in (0, 1]$ by invoking {\WeightMCCore}
$\mathcal{O}(\log_2(1/\delta))$ times, and by using the median of the
non-$\bot$ counts thus returned.  For convenience of exposition, we
use $\eta(t, m, p)$ in the following discussion to denote the
probability of at least $m$ heads in $t$ independent tosses of a
biased coin with $\prob\left[\mathit{heads}\right] = p$.  Clearly, 
$\eta(t, m, p) = \sum_{k=m}^{t} \binom{t}{k} p^{k} (1-p)^{t-k}$.

\setcounter{theorem}{0}
\begin{theorem} \label{theorem:approx}
Given a propositional formula $F$ and parameters $\varepsilon ~(0 <
\varepsilon \le 1)$ and $\delta ~(0 < \delta \le
1)$, suppose {\WeightMC}$(F, \varepsilon, \delta, X, r)$ returns $c$.  Then
$\prob\left[{\left(1+\varepsilon\right)}^{-1}\cdot \weight{\satisfying{F}}) \le c \right.$ $\left.\le
  (1+\varepsilon)\cdot \weight{\satisfying{F}})\right]$ $\ge 1-\delta$.
\end{theorem}
\begin{proof}
Throughout this proof, we assume that {\WeightMCCore} is invoked $t$
times from {\WeightMC}, where $t = \left\lceil 35\log_2 (3/\delta)
\right\rceil$ (see pseudocode for {\ComputeIterCount} in
Section~\ref{sec:algo}).  Referring to the pseudocode of {\WeightMC},
the final count returned is the median of the non-$\bot$
counts obtained from the $t$ invocations of {\WeightMCCore}.  Let
$Err$ denote the event that the median is not in
$\left[(1+\varepsilon)^{-1}\cdot \scaledweight{\satisfying{F}}, (1+\varepsilon)\cdot \scaledweight{\satisfying{F}}\right]$.  Let
``$\#\mathit{non }\bot = q$'' denote the event that $q$ (out of $t$)
values returned by {\WeightMCCore} are non-$\bot$.  Then,
$\prob\left[Err\right]$ $=$ $\sum_{q=0}^t \prob\left[Err \mid
  \#\mathit{non }\bot = q\right]$ $\cdot$
$\prob\left[\#\mathit{non }\bot = q\right]$.

In order to obtain $\prob\left[Err \mid \#\mathit{non }\bot =
  q\right]$, we define a $0$-$1$ random variable $Z_i$, for $1 \le i
\le t$, as follows.  If the $i^{th}$ invocation of {\WeightMCCore}
returns $c$, and if $c$ is either $\bot$ or a non-$\bot$ value that
does not lie in the interval $[(1+\varepsilon)^{-1}\cdot \scaledweight{\satisfying{F}},
  (1+\varepsilon)\cdot \scaledweight{\satisfying{F}}]$, we set $Z_i$ to 1; otherwise, we set
it to $0$.  From Lemma~\ref{thm:almost-approx}, $\prob\left[Z_i =
  1\right] = p < 0.4$.  If $Z$ denotes $\sum_{i=1}^t Z_i$, a necessary
(but not sufficient) condition for event $Err$ to occur, given that
$q$ non-$\bot$s were returned by {\WeightMCCore}, is $Z \ge
(t-q+\lceil q/2\rceil)$.  To see why this is so, note that $t-q$
invocations of {\WeightMCCore} must return $\bot$.  In addition, at
least $\lceil q/2 \rceil$ of the remaining $q$ invocations must return
values outside the desired interval. To simplify the exposition, let
$q$ be an even integer.  A more careful analysis removes this
restriction and results in an additional constant scaling factor for
$\prob\left[Err\right]$.  With our simplifying assumption,
$\prob\left[Err \mid \#\mathit{non }\bot = q\right] \le \prob[Z \ge (t
  - q + q/2)]$ $=\eta(t, t-q/2, p)$.  Since $\eta(t, m, p)$ is a
decreasing function of $m$ and since $q/2 \le t-q/2 \le t$, we have
$\prob\left[Err \mid \#\mathit{non }\bot = q\right] \le \eta(t, t/2,
p)$.  If $p < 1/2$, it is easy to verify that $\eta(t, t/2, p)$ is an
increasing function of $p$.  In our case, $p < 0.4$; hence,
$\prob\left[Err \mid \#\mathit{non }\bot = q\right] \le \eta(t, t/2,
0.4)$.

It follows from the above that $\prob\left[ Err \right]$ $=$
$\sum_{q=0}^t$ $\prob\left[Err \mid \#\mathit{non }\bot \right.$ $\left. = q\right] \cdot\prob\left[\#\mathit{non }\bot = q\right]$ $\le$ $\eta(t, t/2,
0.4)\cdot$ $\sum_{q=0}^t \prob\left[\#\mathit{non }\bot = q\right]$
$=$ $\eta(t, t/2, 0.4)$.  Since $\binom{t}{t/2} \ge \binom{t}{k}$ for
all $t/2 \le k \le t$, and since $\binom{t}{t/2} \le 2^t$, we have
$\eta(t, t/2, 0.4)$ $=$ $\sum_{k=t/2}^{t} \binom{t}{k} (0.4)^{k}
(0.6)^{t-k}$ $\le$ $\binom{t}{t/2} \sum_{k=t/2}^t (0.4)^k (0.6)^{t-k}$
$\le 2^t \sum_{k=t/2}^t (0.6)^t (0.4/0.6)^{k}$ $\le 2^t \cdot 3 \cdot (0.6 \times 0.4)^{t/2}$ $\le 3\cdot(0.98)^t$.
Since $t = \left\lceil 35\log_2 (3/\delta) \right\rceil$, it follows
that $\prob\left[Err\right] \le \delta$.
\end{proof}

\begin{theorem}
Given an oracle for {\SAT}, \WeightMC$(F, \varepsilon, \delta, S, r)$
runs in time polynomial in $\log_2(1/\delta), r, |F|$ and
$1/\varepsilon$ relative to the oracle.
\end{theorem}

\begin{proof}
Referring to the pseudocode for {\WeightMC}, lines \ref{line:weightmc-init-start}--\ref{line:weightmc-init-end} take $\mathcal{O}(1)$ time.
The repeat-until loop in lines \ref{line:weightmc-loop-start}--\ref{line:weightmc-loop} is repeated $t = \left\lceil
35\log_2(3/\delta)\right\rceil$ times. The time taken for each
iteration is dominated by the time taken by {\WeightMCCore}.  Finally,
computing the median in line \ref{line:weightmc-median} takes time linear in $t$.  The proof
is therefore completed by showing that {\WeightMCCore} takes time
polynomial in $|F|,r$ and $1/\varepsilon$ relative to the {\SAT} oracle.

Referring to the pseudocode for {\WeightMCCore}, we find that
{\BoundedWeightSAT} is called $\mathcal{O}(|F|)$ times. Observe that when the loop in {\BoundedWeightSAT} terminates, $\mathrm{w_{min}}$ is such that each $y \in R_F$ whose weight was added to $\mathrm{w_{total}}$ has weight at least $\mathrm{w_{min}}$. Thus since the loop terminates when $\mathrm{w_{total}} / \mathrm{w_{min}} > r \cdot pivot$, it can have iterated at most $(r \cdot pivot) + 1$ times. Therefore each call to {\BoundedWeightSAT} makes at most $(r \cdot pivot) + 1$ calls to the {\SAT} oracle, and takes time polynomial in $|F|$, $r$, and $pivot$ relative to the oracle. Since $\mathit{pivot}$ is in
$\mathcal{O}(1/\varepsilon^2)$, the number of calls to the {\SAT}
oracle, and the total time taken by all calls to {\BoundedWeightSAT} in each
invocation of {\WeightMCCore} is  polynomial in $|F|,r$ and
$1/\varepsilon$ relative to the oracle.  The random choices in lines \ref{line:WeightMCCore-choose-hash}
and \ref{line:weightmccore-choose-alpha} of {\WeightMCCore} can be implemented in time polynomial in $n$
(hence, in $|F|$) if we have access to a source of random bits.
Constructing $F \wedge h(z_1, \ldots z_n) = \alpha$ in line \ref{line:WeightMCCore-call-BoundedWeightSAT} can
also be done in time polynomial in $|F|$.
\end{proof}

\section{Analysis of WeightGen} \label{appendix:weightGen}
 For convenience of analysis, we assume that $\log (\scaledweight{\satisfying{F}}-1)
- \log \mathit{pivot}$ is an integer, where $\mathit{pivot}$ is the
quantity computed by algorithm {\ComputeKappaAndPivot} (see
Section~\ref{sec:algorithm}).  A more careful analysis removes this
assumption by scaling the probabilities by constant factors.  Let us
denote $\log (\scaledweight{\satisfying{F}}-1) - \log \mathit{pivot}$ by $m$.  The expression
used for computing $\mathit{pivot}$ in algorithm
{\ComputeKappaAndPivot} ensures that $\mathrm{pivot} \ge 17$.
Therefore, if an invocation of {\WeightGen} does not return from line \ref{line:weightgen-base-return}
of the pseudocode, then $\scaledweight{\satisfying{F}} \ge 18$.  Note also that the expression
for computing $\kappa$ in algorithm {\ComputeKappaAndPivot} requires
$\varepsilon \ge 1.71$ in order to ensure that $\kappa \in [0, 1)$ can
always be found.

In the case where $\scaledweight{\satisfying{F}} \le 1 + (1+\kappa) pivot$, {\BoundedWeightSAT} returns all witnesses of $F$ and {\WeightGen} returns a perfect weighted-uniform sample on line \ref{line:weightgen-base-return}. So we restrict our attention in the lemmas below to the other case, where as noted above we have $\scaledweight{\satisfying{F}} \ge 18$. The following lemma shows that $q$, computed in line \ref{line:weightgen-q} of the
pseudocode, is a good estimator of $m$.
\begin{lemma}
\label{lm:qBounds}
$\prob [ q-3 \le m \le q] \ge 0.8$
\end{lemma}
\begin{proof}
Recall that in line \ref{line:weightgen-approx-counter} of the pseudocode, an approximate weighted model
counter is invoked to obtain an estimate, $C$, of $\weight{\satisfying{F}}$ with
tolerance $0.8$ and confidence $0.8$.  By the definition of
approximate weighted model counting, we have $\prob[\frac{C}{1.8} \le \weight{\satisfying{F}} \le
(1.8)C] \ge 0.8$. Defining $c = C / \wmax$, we have
$\prob[\log c - \log (1.8) \le \log \scaledweight{\satisfying{F}} \le \log c + \log (1.8)] \ge 0.8$. It follows that
$\prob[\log c - \log (1.8) - \log pivot - \log
(\frac{1}{1-1/\scaledweight{\satisfying{F}}})  \le $  
$\log (\scaledweight{\satisfying{F}}-1) -\log pivot \le \log c - \log pivot + \log (1.8)-\log 
(\frac{1}{1-1/\scaledweight{\satisfying{F}}})] $ $\ge 0.8$. Substituting $q = \lceil \log C - \log \mathrm{w_{max}} + \log 1.8 - \log pivot \rceil = \lceil \log c + \log 1.8 - \log pivot \rceil$, and using the bounds $\mathrm{w_{max}} \le 1$, $\log 1.8 \le 0.85$, and $\log (\frac{1}{1-1/\scaledweight{\satisfying{F}}}) \le 0.12$ (since $\scaledweight{\satisfying{F}} \ge 18$ at line \ref{line:weightgen-approx-counter} of the pseudocode, as noted above), we have $\prob [ q-3 \le m \le q] \ge 0.8$.
\end{proof}

The next lemma provides a lower bound on the probability
of generation of a witness.  Let $w_{i,y,\alpha}$ denote the probability
$\prob\left[\frac{pivot} {1+\kappa} \le
\scaledweight{{\satisfying{F,h,\alpha}}} \leq 1+(1+\kappa) pivot \land h(y)
= \alpha \right]$, with $h \xleftarrow{R} H_{xor}(n, i, 3)$.  The proof of the lemma also provides a lower
bound on $w_{m,y,\alpha}$.

\begin{lemma}\label{lm:lowerBound}
For every witness $y \in \satisfying{F}$, $\prob[\textrm{y is output}] \ge \frac{0.8 (1-e^{-3/2}) \scaledweight{y}}{(1.06+\kappa)(\scaledweight{\satisfying{F}}-1)} $
\end{lemma}
\begin{proof}
Let $U$ denote the event that witness
$y \in R_{F}$ is output by {\WeightGen} on inputs $F$, $\varepsilon$, $r$, and
$X$.  Let $p_{i,y}$ denote the probability that we exit the loop at line \ref{line:weightgen-loop} with a particular value of $i$ and $y \in \satisfying{F,h,\alpha}$, where
$\alpha \in \{0,1\}^{i}$ is the value chosen on line \ref{line:weightgen-alpha}.  Then,
$\prob[U] = \sum_{i=q-3}^{q} \frac{\scaledweight{y}}{\scaledweight{Y}}p_{i,y} \prod_{j =
q-3}^{i-1} (1-p_{j,y})$, where $Y$ is the set returned by
{\BoundedWeightSAT} on line \ref{line:weightgen-bwsat-2}. Let $f_m = \prob [
q-3 \le m \le q]$.  From Lemma~\ref{lm:qBounds}, we know that $f_m \ge
0.8$. From line \ref{line:weightgen-failure}, we also know that
$\frac{1}{1+\kappa} pivot \le \scaledweight{Y} \le
1+(1+\kappa) pivot$. Therefore,
$\prob[U] \ge \frac{\scaledweight{y}}{1+(1+\kappa) pivot}\cdot p_{m,y} \cdot
f_m$. The proof is now completed by showing
$p_{m,y} \ge \frac{1}{2^m}(1-e^{-3/2})$, as then we have
$\prob[U] \ge \frac{0.8(1-e^{-3/2})}{(1+(1+\kappa)pivot)2^{m}} \ge \frac{0.8(1-e^{-3/2})}{(1.06+\kappa)(\scaledweight{R_{F}}|-1)}$. The
last inequality uses the observation that $ 1/ pivot \leq
0.06$.

  To calculate $p_{m,y}$, we first note that since $y \in \satisfying{F}$, the
  requirement ``$y\in R_{F,h,\alpha}$" reduces to ``$ y \in
  h^{-1}(\alpha)$". For $\alpha \in \{0,1\}^{n}$, we define
  $w_{m,y,\alpha} = \prob\left[\frac{pivot}
  {1+\kappa}\right.$ $\left. \le \scaledweight{R_{F,h,\alpha}} \leq
  1+(1+\kappa) \right.$ $\left. pivot \land h(y) = \alpha :
  h \xleftarrow{R} H_{xor}(n, m, 3)\right]$. Then we have $p_{m,y} =
  \Sigma_{\alpha \in \{0,1\}^m} \left(w_{m,y,\alpha} \cdot 2^{-m}\right)$. So to prove the desired bound on $p_{m,y}$ it suffices to show that
  $w_{m,y,\alpha} \ge (1-e^{-3/2})/2^{m}$ for every
  $\alpha \in \{0,1\}^{m}$ and $y \in \{0,1\}^{n}$.

Towards this end, let us 
first fix a random $y$. Now we define an indicator variable $\gamma_{z,
\alpha}$ for every $z \in \satisfying{F}\setminus\{y\}$ such that $\gamma_{z,\alpha} = \scaledweight{z} $ 
if $h(z) = \alpha$, and $\gamma_{z,\alpha} = 0$ otherwise. Let us fix
$\alpha$ and choose $h$ uniformly at random from $H_{xor}(n,m,3)$. The
random choice of $h$ induces a probability distribution on
$\gamma_{z,\alpha} $ such that $\expect[\gamma_{z,\alpha}]
= \scaledweight{z} \prob[\gamma_{z,\alpha} = \scaledweight{z}] = \scaledweight{z} \prob[h(z) = \alpha] = \scaledweight{z} / 2^m$.  Since we have fixed $y$,
and since hash functions chosen from $H_{xor} (n,m,3)$ are $3$-wise
independent, it follows that for every distinct $z_a, z_b \in
\satisfying{F} \setminus \{y\}$, the random variables $\gamma_{z_a,\alpha},
\gamma_{z_b,\alpha}$ are 2-wise independent.
    Let $\Gamma_{\alpha} = \sum_{z \in \satisfying{F} \setminus\{y\}} \gamma_{z,\alpha}$ and $\mu_{\alpha} = 
   \expect[\Gamma_{\alpha}]$. Clearly, $\Gamma_{\alpha} =
   \scaledweight{R_{F,h,\alpha}} -\scaledweight{y}$ and $\mu_{\alpha} = \sum_{z \in \scaledweight{R_{F}} \setminus \{y\}}$ $ \expect[\gamma_{z,\alpha}]$ $= (\scaledweight{R_{F}}-\scaledweight{y})/2^m$. Since $pivot = (\scaledweight{\satisfying{F}} - 1) / 2^m \le (\scaledweight{\satisfying{F}} - \scaledweight{y}) / 2^m$, we have $\prob[\frac{pivot}{1+\kappa} \le \scaledweight{R_{F,h,\alpha}} \le 1+(1+\kappa) pivot]$  $\ge \prob[\frac{\scaledweight{R_{F}}-\scaledweight{y}}{(1+\kappa) 2^m} 
  \le \scaledweight{R_{F,h,\alpha}} \le 1+(1+\kappa) \frac{\scaledweight{R_{F}}-1}{2^m}]$ $\ge \prob[\frac{\scaledweight{R_{F}}-\scaledweight{y}}{2^m(1+\kappa)}  \le 
 \scaledweight{R_{F,h,\alpha}}-\scaledweight{y} \le (1+\kappa) \frac{(\scaledweight{R_{F}}-\scaledweight{y})}{2^m}]$. Since $pivot = \lceil e^{3/2}(1+1/\kappa)^{2}\rceil$ and the variables $\gamma_{z,\alpha}$ are 2-wise independent and in the range $[0,1]$, we may apply Lemma \ref{theorem:chernoff-hoeffding} with $\beta = \kappa / (1 + \kappa)$ to obtain $\prob[\frac{\mathrm{pivot}}{1+\kappa} \le \scaledweight{R_{F,h,\alpha}} \le 1+(1+\kappa) \mathrm{pivot}] \ge 1-e^{-3/2}$. Since $h$ is chosen at 
 random from $H_{xor}(n,m,3)$, we also have $\prob[h(y) = \alpha] = 1/2^m$. It follows that $w_{m,y,\alpha} \ge (1-e^{-3/2})/2^m$.      
\end{proof}

The next lemma provides an upper bound of $w_{i, y, \alpha}$ and $p_{i, y}$.
\begin{lemma}\label{lm:iterProof} 
For $i < m$, both $w_{i,y,\alpha}$ and $p_{i,y}$ are bounded above by
  $\frac{1}{\scaledweight{R_{F}}-1}\frac{1}
  {\left(1-\frac{1+\kappa}{2^{m-i}}\right)^2}$.
\end{lemma}
\begin{proof}
We will use the terminology introduced in the proof of Lemma \ref{lm:lowerBound}.
Clearly, $\mu_{\alpha} = \frac{\scaledweight{R_{F}}-\scaledweight{y}}{2^i}$. Since each
$\gamma_{z,\alpha}$ takes values in $[0,1]$,
$\var\left[\gamma_{z,\alpha}\right] \le
\expect\left[\gamma_{z,\alpha}\right]$.  Therefore, $\sigma^2_{z,
  \alpha}$ $\le \sum_{z \neq y, z \in
  \satisfying{F}}\expect\left[\gamma_{z,\alpha}\right]$ $ \le \sum_{z \in
  \satisfying{F}}\expect\left[\gamma_{z, \alpha}\right]$
 $=\expect\left[\Gamma_\alpha\right] \leq 2^{-m}(\scaledweight{R_{F}}-\scaledweight{y})$. So  $\prob[\frac{pivot}{1+\kappa} \le \scaledweight{R_{F,h,\alpha}} \le 1+(1+\kappa)\mathrm{pivot}] \leq \prob[\scaledweight{R_{F,h,\alpha}}-\scaledweight{y} \leq (1+\kappa)\mathrm{pivot}]$. From Chebyshev's
inequality, we know that $\prob\left[|\Gamma_\alpha - \mu_{z,\alpha}|
  \ge\right.$ $\left.\lambda \sigma_{z,\alpha} \right] \leq 1/\lambda^2$ for every 
  $\lambda >0$. 
  $\prob[\scaledweight{R_{F,h,\alpha}}-\scaledweight{y} \leq 
  (1+\kappa)\frac{(\scaledweight{R_{F}} - \scaledweight{y})}{2^i}]$ $\leq \prob $ 
  $ \left[| (\scaledweight{R_{F,h,\alpha}}-\scaledweight{y}) - 
  \frac{\scaledweight{R_{F}}-1}{2^i} |\right.$ $\left.\geq (1-\frac{1+\kappa}
  {2^{m-i}})\frac{\scaledweight{R_{F}}-\scaledweight{y}}{2^i} \right]$ $\leq \frac{1}{\left(1-\frac{(1+\kappa)}{2^{m-
  i}}\right)^2} \cdot \frac{2^i}{\scaledweight{R_{F}}-1}$.  Since $h$ is chosen at random from 
  $H_{xor}(n,m,3)$, we also have $\prob[h(y) = \alpha] = 1/2^i$. It follows 
  that $w_{i,y,\alpha} \leq \frac{1}{\scaledweight{R_{F}}-1}\frac{1}
  {\left(1-\frac{1+\kappa}{2^{m-i}}\right)^2}$.  The bound for $p_{i,y}$ is easily
obtained by noting that $p_{i,y} = \Sigma_{\alpha \in \{0,1\}^i} \left(w_{i,y,\alpha} \cdot 2^{-i}\right)$.
\end{proof}
\begin{lemma}
\label{lm:upperBound}
For every witness $y \in \satisfying{F}$, $\prob[\textrm{y is output}] \le \frac{(1+\kappa) \scaledweight{y}}{\scaledweight{R_{F}}-1} (2.23+ \frac{0.48}{(1-\kappa)^2})$
\end{lemma}
\begin{proof}
We will use the terminology introduced in the proof of Lemma \ref{lm:lowerBound}. Using $\frac{pivot}{1+\kappa} \le \scaledweight{Y}$, we have $\prob[U] = \sum_{i=q-3}^{q}  \frac{\scaledweight{y}}{\scaledweight{Y}}p_{i,y}$ $ \prod_{j = q-3}^{i} (1-p_{j,y}) \le  \frac{1+\kappa}{pivot} \scaledweight{y} \sum_{i=q-3}^{q} p_{i,y}$. Now we subdivide the calculation of $\prob[U]$ into three cases depending on the value of $m$.\\
\textbf{Case 1 :} $q-3 \le m \le q$.\\
Now there are four values that $m$ can take. 
\begin{enumerate}
\item $m = q-3$. We know that $p_{i,y} \le \prob[h(y) = \alpha] = \frac{1}{2^i}$, so $\prob[U | m = q-3] \le \frac{1+\kappa}{pivot} \cdot \frac{\scaledweight{y}}{2^{q-3}} \frac{15}{8}$. Substituting the values of $pivot$ and $m$ gives $\prob[U | m = q-3] \le \frac{15(1+\kappa)\scaledweight{y}}{8(\scaledweight{R_{F}}-1)}$.

\item $m = q-2$. For $i \in [q-2,q]$ $p_{i,y} \le \prob[h(y) = \alpha] = \frac{1}{2^i}$ Using Lemma \ref{lm:iterProof}, we get $p_{q-3,y} \le \frac{1}{\scaledweight{R_{F}}-1}\frac{1}
  {\left(1-\frac{1+\kappa}{2}\right)^2}$. Therefore, $\prob[U $ $|m = q-2] \leq \frac{1+\kappa}{pivot} \scaledweight{y} \frac{1}{\scaledweight{R_{F}}-1} \frac{4}{\left( 1 - \kappa \right)^2} + \frac{1+\kappa}{pivot} \scaledweight{y} \frac{1}{2^{q-2}} \frac{7}{4}$. Noting that $pivot = \frac{\scaledweight{R_{F}}-1}{2^m} > 10$, we obtain $\prob[U|m = q-2] \leq \frac{(1+\kappa) \scaledweight{y}}{\scaledweight{R_{F}}-1}(\frac{7}{4} + \frac{0.4}{(1-\kappa)^2})$

\item $m = q-1$. For $i \in [q-1,q]$, $p_{i,y} \le \prob[h(y) = \alpha] = \frac{1}{2^i}$. Using Lemma \ref{lm:iterProof}, we get $p_{q-3,y} + p_{q-2,y} \le \frac{1}{\scaledweight{R_{F}}-1} \Big($ $\frac{1}{\left(1-\frac{1+\kappa}{2^2}\right)^2} +  \frac{1}{\left(1-\frac{1+\kappa}{2}\right)^2} \Big) = \frac{1}{\scaledweight{\satisfying{F}} - 1} \left( \frac{16}{(3 - \kappa)^2} + \frac{4}{(1 - \kappa)^2} \right)$. Therefore, $\prob[U|m = q-1] \leq \frac{1+\kappa}{pivot} \scaledweight{y} \Big( \frac{1}{\scaledweight{R_{F}}-1}$ $\left( \frac{16}{(3 - \kappa)^2} + \frac{4}{(1 - \kappa)^2} \right) + \frac{1}{2^{q-1}} \frac{3}{2} \Big)$. Since $pivot = \frac{\scaledweight{R_{F}}-1}{2^m} > 10$ and $\kappa \leq 1$, $\prob[U|m = q-1] \leq \frac{(1+\kappa) \scaledweight{y}}{\scaledweight{R_{F}}-1} (1.9 + \frac{0.4}{(1-\kappa)^2})$.

\item $m = q$. We have $p_{q,y} \leq \prob[h(y) = \alpha] = \frac{1}{2^q}$, and using Lemma \ref{lm:iterProof} we get $p_{q-3,y} + p_{q-2,y} + p_{q-1,y} \leq \frac{1}{\scaledweight{R_{F}}-1}\left(\frac{1}{\left(1-\frac{1+\kappa}{2^3}\right)^2} + \frac{1}{\left(1-\frac{1+\kappa}{2^2}\right)^2} +   \frac{1}
  {\left(1-\frac{1+\kappa}{2}\right)^2}\right) =$ \\ $\frac{1}{\scaledweight{\satisfying{F}} - 1} \left( \frac{64}{(7 - \kappa)^2} + \frac{16}{(3 - \kappa)^2} + \frac{4}{(1 - \kappa)^2} \right)$. So $\prob[U|m = q] \leq \frac{1+\kappa}{pivot} \scaledweight{y} \left(\frac{1}{\scaledweight{R_{F}}-1} \left( \frac{64}{(7 - \kappa)^2} + \frac{16}{(3 - \kappa)^2} + \frac{4}{(1 - \kappa)^2} \right) + 1\right)$. Using $pivot = \frac{\scaledweight{R_{F}}-1}{2^m} > 10$ and $\kappa \le 1$, we obtain $\prob[U|m = q] \leq \frac{(1+\kappa) \scaledweight{y}}{\scaledweight{R_{F}}-1} (1.58+\frac{0.4}{(1-\kappa)^2})$. 
\end{enumerate}
Since $\prob[U | q-3 \leq m \leq q] \leq \max_{q-3 \le i \le q} (\prob[U|m =i])$, we have $\prob[U| q-3 \leq m \leq q] \leq \frac{1+\kappa}{\scaledweight{R_{F}}-1}(1.9+\frac{0.4}{(1-\kappa)^2})$ from the $m = q-1$ case above.
\\
\textbf{Case 2 :} $m < q-3$. Since $p_{i,y} \leq \prob[h(y) = \alpha] = \frac{1}{2^i}$, we have $\prob[U |  m < q-3] \le \frac{1+\kappa}{pivot} \scaledweight{y} \cdot \frac{1}{2^{q-3}} \frac{15}{8}$. Substituting the value of $pivot$ and maximizing $m-q+3$, we get $\prob[U |  m < q-3] \le \frac{15(1+\kappa) \scaledweight{y}}{16(\scaledweight{R_{F}}-1)}$.\\
\textbf{Case 3 :} $m > q$. Using Lemma \ref{lm:iterProof}, we know that $\prob[U | m >q] \leq \frac{1+\kappa}{pivot} \frac{\scaledweight{y}}{\scaledweight{R_{F}}-1}  $ $\sum_{i=q-3}^{q} \frac{1}{\left( 1-\frac{1+\kappa}{2^{m-i}} \right)^2}$. The R.H.S. is maximized when $m = q+1$. Hence $\prob[U | m >q] \leq \frac{1+\kappa}{pivot}$ $ \frac{\scaledweight{y}}{\scaledweight{R_{F}}-1} \times $ $ \sum_{i=q-3}^{q} \frac{1}{\left( 1-\frac{1+\kappa}{2^{q+1-i}} \right)^2}$. Noting that $pivot = \frac{\scaledweight{R_{F}}-1}{2^m} > 10$ and expanding 
  the above summation we have $\prob[U | m > q] \leq \frac{(1+\kappa) \scaledweight{y}}{\scaledweight{R_{F}}-1} \frac{1}{10} \left( \frac{256}{(15-\kappa)^2}+\frac{64}
  {(7 - \kappa)^2}+ \frac{16}{(3 - \kappa)^2} + \frac{2}{(1 - \kappa)^2}\right).$ Using $\kappa \leq 1$ for the first three summation terms, we obtain $\prob[U | m > q] \leq \frac{(1+\kappa) \scaledweight{y}}{\scaledweight{R_{F}}-1} (0.71 + \frac{0.4}{(1-\kappa)^2})$

  Summing up all the above cases, $\prob[U] = \prob[U|m < q-3] \times \prob[m < q-3] + \prob[U|q-3 \le m \le  q] \times \prob[q-3 \le m \le  q] +\prob[U|m > q] \times \prob[m > q]$. From Lemma \ref{lm:qBounds} we have $\prob [m < q-1] \le 0.2$ and $\prob[m > q] \le 0.2$, so $\prob[U] \leq \frac{(1+\kappa) \scaledweight{y}}{\scaledweight{R_{F}}-1} (2.23 + \frac{0.48}{(1-\kappa)^2})$
 
\end{proof}
Combining Lemmas~\ref{lm:lowerBound} and~\ref{lm:upperBound}, the following lemma is obtained.
\begin{lemma} \label{lemma:wtgenuniform}
For every witness $y \in \satisfying{F}$, if $\varepsilon > 1.71$, then\\ 
 $ \frac{\weight{y}}{(1+\varepsilon)\weight{R_{F}}} \le \prob\left[{\WeightGen}(F, \varepsilon, r, X) = y\right] \le $ $ (1+\varepsilon)\frac{\weight{y}}{\weight{R_{F}}}.$
\end{lemma}
\begin{proof}
In the case where $\scaledweight{\satisfying{F}} \le 1 + (1+\kappa) pivot$, the result holds because {\WeightGen} returns a perfect weighted-uniform sample. Otherwise, using Lemmas \ref{lm:lowerBound} and \ref{lm:upperBound} and substituting $(1+\varepsilon) =(1+\kappa)(2.36+\frac{0.51}{(1-\kappa)^2})$ $= \frac{18}{17} (1+\kappa)(2.23+\frac{0.48}{(1-\kappa)^2})$, via the inequality $\frac{1.06+\kappa}{0.8(1-e^{-3/2})} \le \frac{18}{17} (1+\kappa)(2.23+\frac{0.48}{(1-\kappa)^2})$ we have the bounds $ \frac{\scaledweight{y}}{(1+\varepsilon)(\scaledweight{R_{F}}-1)} \le \prob\left[{\WeightGen}(F, \varepsilon, r, X) = y\right] \le \frac{18}{17} (1+\varepsilon)\frac{\scaledweight{y}}{\scaledweight{R_{F}}-1}$. Using $\scaledweight{\satisfying{F}} \ge 18$, we obtain the desired result.
\end{proof}

\begin{lemma} \label{lemma:wtgenSuccProb}
Algorithm {\WeightGen} succeeds (i.e. does not return $\bot$) with probability at least $0.62$.
\end{lemma}
\begin{proof}
If $\scaledweight{R_{F}} \le 1+(1+\kappa)\mathrm{pivot}$, the theorem holds trivially. Suppose $\scaledweight{R_{F}} > 1+(1+\kappa)\mathrm{pivot}$ and let $P_{\mathrm{succ}}$ denote the probability that a run of the algorithm succeeds. Let 
$p_i$ with $q-3 \le i \le q$ denote the conditional probability that {\WeightGen} 
($F$, $\varepsilon$, $r$, $X$) terminates in iteration $i$ of the repeat-until loop (lines \ref{line:weightgen-loop-start}--\ref{line:weightgen-loop}) 
with $\frac{pivot}{1+\kappa} \le \scaledweight{R_{F,h,\alpha}} \le 
1+(1+\kappa) pivot$, given that $\scaledweight{R_{F}} > 1+(1+\kappa) pivot$.
Then $P_{\mathrm{succ}} = \sum_{i=q-3}^{q} p_{i} \prod_{j = q-3}^{i} (1-p_{j})$. Letting $f_m = \prob [ q-3 \le m \le q]$, by Lemma \ref{lm:qBounds} we have $P_{\mathrm{succ}} \ge p_{m} f_m \ge 0.8p_{m}$. The theorem is now proved by using Lemma \ref{theorem:chernoff-hoeffding} to show 
that $p_m \ge 1-e^{-3/2} \ge 0.776$.\\
   For every $y \in \{0,1\}^{n}$ and $\alpha \in \{0,1\}^{m}$, define an indicator variable $\nu_{y,\alpha}$ as follows: $\nu_{y,\alpha} = \scaledweight{y}$ if $h(y) = \alpha$, and $\nu_{y,\alpha} = 0$ otherwise. Let us fix $\alpha$ and $y$ and choose $h$ uniformly at random from $H_{xor}(n ,m,3)$. The random choice of $h$ induces a probability distribution on $\nu_{y,\alpha}$, such that $\prob[\nu_{y,\alpha} = \scaledweight{y}] = \prob[h(y) = \alpha] = 2^{-m}$ and $\expect[\nu_{y,\alpha}] = \scaledweight{y} \prob[\nu_{y,\alpha} = 1] = 2^{-m} \scaledweight{y}$. In addition 3-wise independence of hash functions chosen from $H_{xor}(n ,m,3)$ implies that for every distinct $y_a,y_b,y_c \in \satisfying{F}$, the random variables $\nu_{y_a,\alpha}, \nu_{y_b,\alpha}$ and $\nu_{y_c,\alpha}$ are 3-wise independent. 

Let $\Gamma_\alpha = \sum_{y \in \satisfying{F}} \nu_{y, \alpha}$ and
$\mu_\alpha = \expect\left[\Gamma_\alpha\right]$.  Clearly,
$\Gamma_\alpha = \scaledweight{R_{F, h, \alpha}}$ and $\mu_\alpha = \sum_{y \in
  \satisfying{F}} \expect\left[\nu_{y, \alpha}\right] = 2^{-m}\scaledweight{R_{F}}$.
Since $pivot = \lceil e^{3/2}(1+1/\epsilon)^{2}\rceil $, we have $2^{-m} \scaledweight{R_{F}} \ge e^{3/2} (1 + 1/\varepsilon)^2$, and so using Lemma \ref{theorem:chernoff-hoeffding} with $\beta = \kappa / (1 + \kappa)$ we obtain $\prob\left[\frac{\scaledweight{R_{F}}}{2^m}.\left(1-\frac{\kappa}{1+\kappa}\right) \leq
  \scaledweight{R_{F,h,\alpha}}\right.$ $\left.\leq
  (1+\frac{\kappa}{1+\kappa})\frac{\scaledweight{R_{F}}}{2^m} \right] > 1- e^{-3/2}$.
 Simplifying and noting that $\frac{\kappa}{1+\kappa} <
\kappa$ for all $\kappa > 0$, we have
$\prob\left[(1+\kappa)^{-1}\cdot \frac{\scaledweight{R_{F}}}{2^m} \leq
  \scaledweight{R_{F,h,\alpha}}\right.$ $\left.\leq (1+ \kappa)\cdot \frac{\scaledweight{R_{F}}}{2^m}
  \right] > 1- e^{-3/2}$. Also, $\frac{pivot}{1+\kappa} = \frac{1}{1+\kappa}\frac{\scaledweight{R_{F}}-1}{2^m} \le \frac{\scaledweight{R_{F}}}{(1+\kappa)2^m}$ and $1+(1+\kappa)pivot = 1+ \frac{(1+\kappa)(\scaledweight{R_{F}}-1)}{2^m} \ge \frac{(1+\kappa)\scaledweight{R_{F}}}{2^m}$. Therefore, $p_m = \prob[\frac{pivot}{1+\kappa} \le \scaledweight{R_{F,h,\alpha}} \le 1+(1+\kappa)pivot] \ge$ $\prob\left[(1+\kappa)^{-1}\cdot \frac{\scaledweight{R_{F}}}{2^m} \right.$ $\left.\leq
  \scaledweight{R_{F,h,\alpha}}\right.$ $\left.\leq (1+ \kappa)\cdot \frac{\scaledweight{R_{F}}}{2^m}
  \right] \ge 1-e^{-3/2}$. 
\end{proof}

By combining Lemmas~\ref{lemma:wtgenuniform} and \ref{lemma:wtgenSuccProb}, we get the following:
\begin{theorem}
Given a CNF formula $F$, tolerance $\varepsilon > 1.71$, tilt bound $r$, and independent support $S$, for every $y \in R_F$ we have
 $ \frac{\weight{y}}{(1+\varepsilon)\weight{R_{F}}} \le \prob\left[{\WeightGen}(F, \varepsilon, r, X) = y\right] \le $ $ (1+\varepsilon)\frac{\weight{y}}{\weight{R_{F}}}.$
 Also, {\WeightGen} succeeds (i.e. does not return $\bot$) with probability at least $0.62$.
\end{theorem}

\begin{theorem}
Given an oracle for {\SAT}, $\WeightGen(F,\varepsilon,r,S)$ runs in time polynomial in $ r, |F|$ and
$1/\varepsilon$ relative to the oracle.
\end{theorem}
\begin{proof}
Referring to the pseudocode for {\WeightGen}, the runtime of the algorithm is bounded by the runtime of the constant number (at most 5) of calls to {\BoundedWeightSAT} and one call to {\WeightMC} (with parameters $\delta=0.2,\varepsilon=0.8$). As shown in Theorem \ref{thm:wtmcApproximation}, the call to {\WeightMC} can be done in time polynomial in $|F|$ and $r$ relative to the oracle.
Every invocation of {\BoundedWeightSAT} can be implemented by at most $(r \cdot pivot) + 1$ calls to a {\SAT} oracle (as in the proof of Theorem \ref{thm:wtmcComplexity}), and the total time taken by all calls to {\BoundedWeightSAT} is polynomial in $|F|, r$ and $pivot$ relative to the oracle. Since $pivot = \mathcal{O}(1/\varepsilon^2)$, the runtime of {\WeightGen} is polynomial in $r$, $|F|$ and $1/\varepsilon$ relative to the oracle.
\end{proof}

\section{Analysis of Partitioned WeightMC}
\label{appendix:partitionedweightmc}

\begin{theorem}
\label{theorem:partitioned-approx}
If $\PartitionedWeightMC(F, \varepsilon, \delta, S, L, H)$ returns $c$ 
(and all arguments are in the required ranges), then
\[
\Pr \left[ c \ne \bot \land (1 + \varepsilon)^{-1} \weight{\satisfying{F}} \le c \le (1 + \varepsilon) \weight{\satisfying{F})} \right] \ge 1 - \delta .
\]
\end{theorem}
\begin{proof}
For future reference note that since $N \ge 1$ and $\delta < 1$, we have $(1 - \delta')^{N} = (1 - \delta/N)^{N} \ge 1 - \delta$. Define $G_m = F \land (H / 2^m < \weight{X} \le H / 2^{m-1})$, the formula passed to {\WeightMC} in iteration $m$. Clearly, we have $\weight{\satisfying{F}} = \sum_{m = 1}^{N} \weight{\satisfying{G_m}}$. Since $\weight{\cdot}$ is poly-time computable, the {\NP} oracle used in {\WeightMC} can decide the satisfiability of $G_m$, and so {\WeightMC} will return a value $d_m$. Now since $H / 2^m$ and $H / 2^{m-1}$ are lower and upper bounds respectively on the weights of any solution to $G_m$, by Theorem \ref{theorem:approx} we have
\[
\Pr \left[ d_m \ne \bot \land (1 + \varepsilon)^{-1} \weight{\satisfying{G_m}} \le d_m \le (1 + \varepsilon) \weight{\satisfying{G_m}} \right] \ge 1 - \delta'
\]
for every $m$, and so
\begin{align*}
& \Pr \left[ c \ne \bot \land (1 + \varepsilon)^{-1} \weight{\satisfying{F}} \le c \le (1 + \varepsilon) \weight{\satisfying{F}} \right] \\
& = \Pr \left[ c \ne \bot \land (1 + \varepsilon)^{-1} \sum_m \weight{\satisfying{G_m}} \le c \le (1 + \varepsilon) \sum_m \weight{\satisfying{G_m}} \right] \\
& \ge \left( 1 - \delta' \right)^{N} \ge 1 - \delta
\end{align*}as desired.
\end{proof}

\begin{theorem}
\label{theorem:partitioned-runtime}
With access to an {\NP} oracle, the runtime of 
$\PartitionedWeightMC(F, \varepsilon, \delta, S, L, H)$ is polynomial in 
$|F|$, $1/\varepsilon$, $\log (1/\delta)$, and $\log r = \log (H / L)$.
\end{theorem}
\begin{proof}
Put $r = H / L$. By Theorem \ref{thm:wtmcComplexity}, each call to {\WeightMC} runs in time polynomial in $|G|$, $1/\varepsilon$ and $\log (1/\delta')$ (the tilt bound is constant). Clearly $|G|$ is polynomial in $|F|$. Since $\delta' = \delta / N$ we have $\log(1/\delta') = \log (N / \delta) = \mathcal{O} ( \log ( (\log_k r) / \delta ) ) = \mathcal{O} ( \log \log r + \log(1/\delta) )$.
Therefore each call to {\WeightMC} runs in time polynomial in $|F|$, $1/\varepsilon$, $\log(1/\delta)$, and $\log \log r$. Since there are $N = \mathcal{O}(\log r)$ calls, the result follows.
\end{proof}

\end{document}